\documentclass[10pt,a4paper,logo]{paper}

\usepackage[all]{hypcap}

\usepackage[authoryear, round]{natbib}

\usepackage[section]{placeins}
\usepackage{nicefrac}       
\usepackage{multirow}
\usepackage{cleveref}
\IfFileExists{dsfont.sty}{\usepackage{dsfont}}{}
\providecommand{\mathds}[1]{\mathbb{#1}}
\usepackage{makecell}
\usepackage{subcaption}
\IfFileExists{bbm.sty}{\usepackage{bbm}}{}

\usepackage{wrapfig}
\usepackage{mathtools,xparse}
\usepackage{array}
\usepackage{arydshln}
\usepackage{scalerel}
\usepackage{tablefootnote}
\usepackage{tocloft}
\usepackage{adjustbox}
\usepackage{setspace}
\usepackage[most,skins,theorems]{tcolorbox}

\newtcolorbox{promptbox}{
  colback=gray!5,
  colframe=gray!70,
  boxrule=0.8pt,
  arc=2mm,
  left=10pt,
  right=10pt,
  top=10pt,
  bottom=10pt,
  width=\linewidth,
  fontupper=\ttfamily\footnotesize,
  enhanced,
}

\newboolean{showsection}
\setboolean{showsection}{true}
\makeatletter
\@namedef{ver@everyshi.sty}{}
\makeatother

\definecolor{custom_green}{rgb}{0.0, 0.5, 0.0}
\definecolor{custom_red}{rgb}{1.0, 0.01, 0.24}
\definecolor{custom_blue}{HTML}{C9DAF7}
\definecolor{custom_purple}{HTML}{D9D1E9}
\definecolor{title_blue}{HTML}{204899}
\definecolor{cite_blue}{HTML}{044dc1}
\definecolor{cite_purple}{HTML}{7406a7}

\hypersetup{
    colorlinks = true,
    citecolor = {MidnightBlue},
    linkcolor = {BrickRed},
    urlcolor = {RoyalBlue},
}

\definecolor{blanchedalmond}{rgb}{1.0, 0.92, 0.8}
\definecolor{carmine}{rgb}{0.59, 0.0, 0.09}
\definecolor{lightblue}{rgb}{0.22,0.45,0.70}%

\newtheorem{theorem}{Theorem}[section]

\newtheorem{lemma}[theorem]{Lemma}

\newtheorem{proposition}[theorem]{Proposition}
\newtheorem{definition}[theorem]{Definition}

\renewcommand{\mathbf}{\boldsymbol}

\makeatletter
\def\Ddots{\mathinner{\mkern1mu\raise\p@
\vbox{\kern7\p@\hbox{.}}\mkern2mu
\raise4\p@\hbox{.}\mkern2mu\raise7\p@\hbox{.}\mkern1mu}}
\makeatother

\numberwithin{equation}{section}

\definecolor{amaranth}{rgb}{0.9, 0.17, 0.31}
\definecolor{antiquebrass}{rgb}{0.8, 0.58, 0.46}
\definecolor{antiquefuchsia}{rgb}{0.57, 0.36, 0.51}
\definecolor{chromeyellow}{rgb}{0.31, 0.47, 0.26}

\newcommand{\1}{\mathds 1}

\usepackage{amsmath,amsfonts,bm}

\def\eqref#1{equation~\ref{#1}}

\def\1{\bm{1}}

\DeclareMathAlphabet{\mathsfit}{\encodingdefault}{\sfdefault}{m}{sl}
\SetMathAlphabet{\mathsfit}{bold}{\encodingdefault}{\sfdefault}{bx}{n}

\makeatletter
\def\mathcolor#1#{\@mathcolor{#1}}
\def\@mathcolor#1#2#3{%
  \protect\leavevmode
  \begingroup
    \color#1{#2}#3%
  \endgroup
}
\makeatother

\Crefformat{equation}{#2Eq.\;(#1)#3}
\Crefformat{figure}{#2Figure #1#3}
\Crefformat{assumption}{#2Assumption #1#3}
\Crefname{assumption}{Assumption}{Assumptions}

\usepackage{crossreftools}
\makeatletter
\renewcommand\footnoterule{%
  \kern -3\p@
  \hrule \@width 2in \kern 2.6\p@
}
\makeatother

\titlespacing*{\paragraph}
  {0pt}
  {0.5\baselineskip}
  {0.5em}

\title{KV4-RL: Unlocking Faster Asynchronous Rollouts for RLVR with INT4 KV Cache Quantization}

\reportnumber{}

\author[1]{Sihwan Park}
\author[2]{Sung-Yub Kim}
\author[1]{Doohyuk Jang}
\author[1,3,\footnotesize{$\dagger$}]{Eunho Yang}

\affil[1]{KAIST}
\affil[2]{NAVER Cloud}
\affil[3]{AITRICS}

\correspondingauthor{Correspondence to: \email{\{sihwan.park, eunhoy\}@kaist.ac.kr}.}

\begin{abstract}
\paragraph{Abstract.} Asynchronous rollout improves the efficiency of reinforcement learning with verifiable rewards (RLVR) by reducing under-utilization during long-tail generation. Sustained concurrency also concentrates context processing into fewer decode steps, increasing the KV workload per step. We characterize this \emph{KV concentration} and find that KV cache quantization delivers substantially greater acceleration than weight-only and weight-activation quantization in our matched-workload comparisons. However, INT4 KV rollout can degrade learning even with token-level truncated importance sampling (TIS). We propose \emph{KV4-RL}, which combines calibration-free INT4 KV quantization with selective INT8 protection of attention sinks and recent context. Protecting a bounded number of positions reduces sampler mismatch while retaining most of the speedup from INT4. Across 400 iterations of Qwen3-8B-Base GRPO training, KV4-RL achieves $1.80\times$ cumulative rollout speedup and $1.27\times$ cumulative full-iteration speedup over BF16 rollout, reaching $1.97\times$ rollout speedup over the final 50 iterations. Under a common BF16 evaluator, its final six-benchmark mean accuracy is within 0.5 percentage points of BF16-rollout training for both Qwen3-4B-Base and Qwen3-8B-Base.
\end{abstract}

\begin{document}

\maketitle

\section{Introduction}
\label{sec:introduction}

Reinforcement learning with verifiable rewards (RLVR) has become a central approach for training reasoning models. Methods such as Group Relative Policy Optimization (GRPO) learn from multiple responses per prompt~\citep{grpo,gspo,dapo}, making rollout generation a major training cost. Reasoning responses can span thousands of tokens and vary substantially in length~\citep{areal,april}. In conventional wave-based collection, shorter responses finish early while a few long responses continue decoding, causing the active batch to collapse. Asynchronous systems mitigate this long-tail inefficiency by admitting new requests~\citep{areal,rollflash}, while partial-rollout schemes carry unfinished responses across policy updates~\citep{kimi,april}. We study quota-based partial collection, which combines these mechanisms to sustain decode concurrency.

Sustaining concurrency improves utilization, but also changes the workload processed at each decode step. For a fixed set of responses, asynchronous scheduling preserves the total generated tokens and logical attention work, but advances more responses together and concentrates the same KV processing into fewer sequential steps. We call this increase in the average KV workload processed per decode step \emph{KV concentration}. Response-length variability strengthens this effect because longer responses remain active for more steps and are therefore overrepresented in the active batch. In a matched Qwen3-4B-Base trace, asynchronous collection increases mean nominal resident KV volume by $7.82\times$ and mean decode-step latency by $4.52\times$. The larger active batch still yields $1.83\times$ higher decode-service throughput, but its per-step cost offsets a substantial part of the concurrency gain. At longer contexts, KV capacity can also restrict request admission. Fully realizing the benefit of asynchronous rollout therefore requires processing its concentrated workload efficiently.

This workload shift changes which quantization target provides the greatest rollout acceleration. Prior work has primarily quantized model weights or activations~\citep{flashrl,qerl,qurl}, with recent work also incorporating FP8 KV caches into broader low-precision stacks~\citep{fp8rl}. Under sustained concurrency, weight reads are increasingly amortized across the active batch, whereas attention continues to process the context of every response. Our roofline analysis identifies KV reads as the largest idealized service demand in the studied workloads, and matched-workload measurements show substantially greater decode acceleration from KV quantization than from the evaluated weight-only and weight-activation configurations. Cache compression can additionally sustain higher concurrency by fitting more resident context within the same HBM budget. Together, these results identify KV cache quantization as an effective target for asynchronous rollout.

Aggressive KV quantization nevertheless creates a distinct challenge for RL training since rollouts generate the samples used for subsequent policy updates. Numerical perturbations can alter the prefixes visited by the sampler, group-relative advantages, and ultimately the learned policy. Token-level truncated importance sampling (TIS) corrects probability ratios at sampled prefixes but does not remove all of these effects. Under the same asynchronous GRPO and TIS recipe, 8-bit KV rollout remains close to BF16-rollout training, whereas 4-bit KV rollout misses the later improvement of BF16 and finishes 5.24 percentage points lower on MATH-500 avg@8. Uniformly improving 4-bit fidelity through finer groups or more expensive reconstruction would add overhead throughout the context. The central design problem is therefore to recover sufficient rollout sampler fidelity without giving up the context-scaling benefit of 4-bit.

We address this trade-off with \textbf{KV4-RL}, a calibration-free INT4 KV cache for asynchronous RLVR rollout. KV4-RL stores the full context in paged INT4 format and protects a bounded set of attention-sink and recent positions using an INT8 side cache and higher-fidelity probability-value arithmetic. Its additional storage and attention work remain bounded independently of context length. This selective-precision design sharply reduces shared-prefix conditional KL while retaining most of the INT4 decode acceleration, and recovers the late-stage learning improvement missing from INT4 KV rollout. End-to-end experiments on Qwen3-4B-Base and Qwen3-8B-Base show that KV4-RL preserves final-policy quality close to BF16-rollout training while substantially accelerating rollout and full training iterations, with larger rollout gains as responses become longer. Our contributions are as follows:
\begin{itemize}[leftmargin=7mm]
\item We identify and characterize \emph{KV concentration} as a structural consequence of asynchronous rollout. By concentrating context processing into fewer decode steps, asynchronous scheduling shifts the rollout workload toward KV processing. Through matched-workload comparisons, we show that KV cache quantization provides substantially greater decode acceleration than the tested weight-only and weight-activation alternatives.

\vspace{5pt}
\item To our knowledge, we present the first study of sub-8-bit KV cache quantization for RLVR rollout. We show that INT4 KV rollout can degrade learning despite token-level TIS, and propose \emph{KV4-RL}, which uses bounded INT8 protection to reduce sampler mismatch while retaining most of the INT4 execution benefit.

\vspace{5pt}
\item We validate KV4-RL through full-parameter GRPO training on Qwen3-4B-Base and Qwen3-8B-Base. Under a common BF16 evaluator, its final six-task mean accuracy remains within 0.5 percentage points of BF16-rollout training at both model sizes. On Qwen3-8B-Base, KV4-RL achieves $1.80\times$ cumulative rollout speedup and $1.27\times$ cumulative full-iteration speedup, with rollout speedup reaching $1.97\times$ over the final 50 iterations.
\end{itemize}
\section{Related Work}
\label{sec:related-work}

\paragraph{Asynchronous and partial rollout.}
RL training systems mitigate rollout bottlenecks through flexible placement and scheduling of generation and policy optimization~\citep{verl,areal,rollflash}. One line of work overlaps rollout collection with learner execution, increasing hardware utilization while controlling the policy staleness introduced by asynchronous updates~\citep{areal,staleflow}. Another line addresses long-tail generation through partial rollout, allowing completed responses to enter training while unfinished responses continue across subsequent updates~\citep{kimi,april}. These approaches primarily study how scheduling reduces idle time or exposes more parallelism. We examine a complementary system consequence. By keeping more responses active, asynchronous collection concentrates context processing into fewer and larger decode steps, changing both the execution bottleneck and the relative value of different quantization targets. Our experiments focus on colocated, quota-based partial collection, where generation and optimization alternate on the same GPUs and unfinished prefixes are carried across policy updates.

\paragraph{Quantized reinforcement learning.}
Quantized RL methods reduce rollout cost by lowering the precision of model execution. Existing approaches predominantly target weights or activations, using weight-only or weight--activation quantization with either low-rank adaptation or full-parameter policy updates~\citep{qerl,flashrl,qurl}. Since the resulting sampler may differ numerically from the learner policy, these methods combine quantized rollout with importance correction, modified policy objectives, or other mechanisms that improve robustness to rollout--training mismatch~\citep{tis,flashrl,qurl}. A related direction instead aligns the numerical execution paths used during generation and training, or introduces compensation on both sides of the update~\citep{jetrl,quads}. FP8-RL extends low-precision execution to a broader training stack that also includes FP8 KV caches~\citep{fp8rl}. These studies establish that low-precision rollout can preserve learning when sampler mismatch remains sufficiently controlled. Our work investigates a different numerical target and operating regime: sub-8-bit KV cache quantization under the sustained concurrency and long contexts produced by asynchronous rollout. We examine both the execution efficiency of rollout KV quantization and its effect on learned-policy quality, while keeping learner computation and evaluation in BF16 to isolate the rollout precision.

\paragraph{KV cache compression and selective precision.}
KV cache compression reduces the storage and memory traffic associated with attention during autoregressive generation. Distribution-aware quantizers exploit the distinct statistical properties of keys and values~\citep{kivi,kvquant}, while rotation-based methods mix channels to reduce sensitivity to outliers before quantization~\citep{quarot,saw}. To translate storage reduction into practical speedup, efficient attention implementations combine packed KV representations with low-bit dot products and fused reconstruction~\citep{qserve,bitdecoding,saw}. Selective-precision methods complement these techniques by retaining influential initial positions~\citep{streamingllm,kvsink}, recent context~\citep{kivi,skvq}, or dynamically selected regions at higher precision~\citep{pmkvq}. Most of these methods are designed and evaluated for inference, where numerical perturbations affect the generated output of an individual request. RL rollout introduces an additional constraint because those outputs become training data and can influence subsequent policy updates. KV4-RL builds on these compression and selective-precision principles, but utilizes them to balance execution efficiency against sampler fidelity during repeated RLVR training.

\begin{figure}[t]
  \centering
  \includegraphics[width=0.9\linewidth]{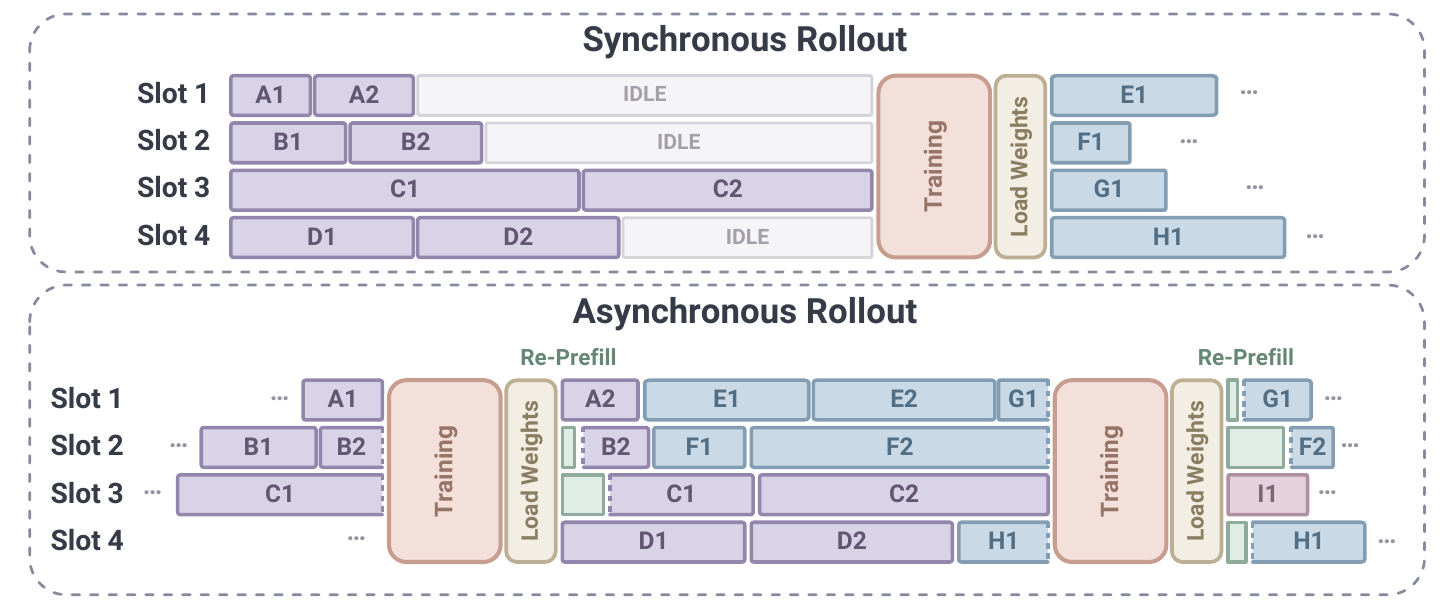}
  \vspace{-5pt}
  \caption{Illustration of synchronous and quota-based partial rollout with $T=4$ prompt groups, $G=2$ responses per prompt, and four logical decode slots. Each row represents one position in the active decode batch. A--D denote distinct prompts, and A1/A2 denote the two responses from prompt A. Colors distinguish the update cohorts shown in the schematic but do not indicate assignments fixed in advance. Training, weight loading, and re-prefill boundaries are marked explicitly.}
  \label{fig:app-async-rollout}
  \vspace{-10pt}
\end{figure}

\section{Preliminaries}
\label{sec:preliminaries}

\paragraph{Quota-based partial rollout.}
Each training iteration requires $T$ completed prompt groups, with $G$ responses per group. As shown in Fig.~\ref{fig:app-async-rollout}, the synchronous baseline launches the groups assigned to the current wave and waits until all $TG$ responses have completed. Its active batch therefore contracts as shorter responses finish, leaving decode capacity idle until the longest response terminates and the policy update begins. Our asynchronous collector instead maintains an oversubscribed request pool and admits new responses whenever decode capacity becomes available. Responses can begin in one collection interval and continue across a policy update, while work that may contribute to a later update can start before the current quota is reached. The scheduler does not assign a fixed set of prompt groups to each update in advance. Collection can stop once $T$ prompt groups have each completed all $G$ responses, provided that no unfinished group has reached the configured age limit. The oldest $T$ completed groups are selected for training, additional completed groups remain in the pool, and unfinished groups retain their token prefixes for a later interval.

After a policy update, the KV caches of carried prefixes cannot be reused since they were computed under the staled policy. The retained token sequences are therefore re-prefilled under the updated policy to reconstruct their KV caches before decoding resumes. Re-prefill does not resample or modify the retained tokens, and the probability originally recorded for each sampled token remains unchanged. A carried response can consequently contain tokens sampled under multiple policy versions, while its KV cache is reconstructed under the current rollout policy whenever generation resumes. In our colocated configuration, rollout generation and policy optimization alternate on the same GPUs. Thus, \emph{asynchronous rollout} refers to responses progressing across collection boundaries and policy updates, rather than concurrent rollout and learner execution.

\paragraph{RLVR and GRPO.}
RLVR optimizes an autoregressive policy $\pi_\theta$ using rewards obtained by verifying generated responses. Given a prompt, GRPO samples $G$ responses and assigns response $y_i$ the group-normalized advantage~\citep{grpo,dapo,gspo}
\begin{equation}
    \hat A_i
    =
    \frac{R_i-\bar R}{s_R+\delta},
    \label{eq:prelim-advantage}
\end{equation}
where $R_i$ is the response reward, $\bar R$ and $s_R$ are the group reward mean and standard deviation, and $\delta>0$ stabilizes normalization. The response-level advantage is shared across all tokens of $y_i$. Let $h_{i,j}$ denote the prompt and response prefix preceding token $y_{i,j}$. Given the frozen pre-update policy $\pi_{\mathrm{old}}$, the token-level likelihood ratio and PPO-style clipped surrogate~\citep{ppo} are
\begin{equation}
    r_{i,j}(\theta)
    =
    \frac{\pi_\theta(y_{i,j}\mid h_{i,j})}
         {\pi_{\mathrm{old}}(y_{i,j}\mid h_{i,j})},\quad
    \mathcal S_{i,j}(\theta)
    =
    \min\!\left\{
        r_{i,j}(\theta)\hat A_i,\,
        \operatorname{clip}
        \!\left(r_{i,j}(\theta),1-\epsilon,1+\epsilon\right)
        \hat A_i
    \right\}.
\label{eq:prelim-grpo}
\end{equation}
Because responses may span multiple updates, $\pi_{\mathrm{old}}$ need not be the policy that originally generated every collected token. Our implementation uses DAPO-style token-level averaging~\citep{dapo}, KL regularization toward a fixed initial policy $\pi_{\mathrm{ref}}$, and dual clipping for negative advantages. Table~\ref{tab:app-training} gives the complete optimization configuration.

\paragraph{Quantization and rotations.}
A $b$-bit affine quantizer represents a value $x$ using an integer code $c$ as
\begin{equation}
    \hat x
    =
    \Delta c+m,
    \qquad
    c\in\{0,\ldots,2^b-1\},
    \label{eq:prelim-affine-quantization}
\end{equation}
where $\Delta>0$ is the scale and $m$ is an offset. Group-wise quantization shares these parameters across a block of values, making its resolution sensitive to outliers within the group. Hadamard rotations mitigate this sensitivity by mixing coordinates before quantization~\citep{quarot}. For a normalized Hadamard matrix $H$ satisfying $HH^\top=I$, applying the same rotation to queries and keys preserves their unquantized attention scores:
\begin{equation}
    (QH)(KH)^\top
    =
    QHH^\top K^\top
    =
    QK^\top.
    \label{eq:prelim-hadamard}
\end{equation}
Keys can therefore be stored in the rotated basis if queries are transformed into the same basis before attention. We use W$w$A$a$ to denote $w$-bit weights and $a$-bit activations, and KV$b$ to denote a $b$-bit KV cache. For example, W4A16 uses 4-bit weights with 16-bit activations, while KV4 uses a 4-bit KV cache. Our KV-only configurations retain BF16 model weights and linear-layer activations, and learner computation remains in BF16.

\paragraph{Sampler mismatch and token-level correction.}
Let $q_{i,j}$ be the probability recorded by the rollout sampler when it generated token $y_{i,j}$. It can differ from $\pi_{\mathrm{old}}(y_{i,j}\mid h_{i,j})$ since the token was generated under an earlier policy version, a different numerical execution path, or both. Token-level truncated importance sampling (TIS) applies the correction~\citep{tis,qurl}
\begin{equation}
    c_{i,j}=\min\!\left\{C,\,\frac{\pi_{\mathrm{old}}(y_{i,j}\mid h_{i,j})}{q_{i,j}}\right\},\quad\widehat{\mathcal J}_{\mathrm{TIS}}(\theta)=\frac{1}{N_{\mathrm{tok}}}\sum_{i,j}c_{i,j}\,\mathcal S_{i,j}(\theta),
\label{eq:prelim-tis}
\end{equation}
where $C\geq1$ is the truncation threshold and $N_{\mathrm{tok}}$ is the number of valid response tokens. The ratio $r_{i,j}(\theta)$ in Eq.~\ref{eq:prelim-grpo} describes the learner update relative to $\pi_{\mathrm{old}}$, whereas $c_{i,j}$ corrects the difference between $\pi_{\mathrm{old}}$ and the sampler that generated the token. For carried responses, $q_{i,j}$ remains the probability recorded under the policy version and numerical path used at the original generation step. Training samples at temperature one without top-$k$ or nucleus truncation, so the recorded softmax probabilities match the sampling distribution. TIS corrects token probabilities at the sampled prefixes, subject to truncation, but does not alter which prefixes the sampler visits or the responses used to compute the group-normalized advantage in Eq.~\ref{eq:prelim-advantage}. Numerical sampler fidelity can therefore remain important even when token-level correction is applied.
\section{KV Concentration in Asynchronous Rollout}
\label{sec:3}

Quota-based partial rollout avoids the low-utilization tail of synchronous generation by continuously filling available decode slots. This improvement, however, does not eliminate the underlying generation work. Instead, it changes how that work is distributed across decode steps, raising a central question for accelerating asynchronous rollout: \emph{Which component of the decode workload becomes the dominant cost when high concurrency is sustained across long responses?}

We answer this question by first characterizing how asynchronous scheduling concentrates KV processing into fewer decode steps and validating this behavior in practical rollout traces. We then compare candidate quantization targets and show that KV cache quantization most directly addresses the resulting bottleneck.

\subsection{Asynchronous Rollout Concentrates KV Work}
\label{sec:3.1}

Consider a fixed realized set of responses with prompt lengths $P_i$ and generation lengths $L_i$. Scheduling does not change the total number of generated tokens, $\sum_i L_i$. Likewise, a response with prompt length $P$ and generation length $L$ accesses $Q(P,L)=\sum_{g=0}^{L-1}(P+g)=PL+L(L-1)/2$ logical KV positions during decoding, regardless of when those accesses occur. Asynchronous scheduling therefore does not reduce the total logical attention work required to generate this fixed response set. Rather, by advancing more responses at each decode step, it concentrates the same work into fewer sequential steps. We refer to this effect as \emph{KV concentration}.\footnote{This is a decode-only accounting for a fixed realized response set. Responses are followed to completion across quota boundaries, while prefill and carried-prefix reconstruction are excluded. End-to-end comparisons additionally reflect policy-dependent response lengths.}

We next formalize this intuition to show that KV concentration is not merely an empirical artifact of a particular rollout trace. Let $\mathcal R_t$ denote the active responses at decode step $t$ and let $B_t=|\mathcal R_t|$. If response $i$ has a retained context of length $\ell_i(t)$, the mean active context length is $\bar\ell_t=B_t^{-1}\sum_{i\in\mathcal R_t}\ell_i(t)$, and the corresponding nominal resident KV volume is $V_t=\kappa B_t\bar\ell_t$, where $\kappa$ denotes the number of bytes required to store K and V for one context position across all layers under the given precision and parallelism. This quantity describes the logical volume of active contexts rather than allocated pool capacity or physical HBM traffic. For each schedule $s\in\{\mathrm{sync},\mathrm{async}\}$, we use $\overline\ell_s$ and $\overline V_s$ to denote the corresponding long-run averages over decode steps. These averages weight decode calls equally and exclude prefill and idle intervals.

\paragraph{A structural consequence of continuous refill.}
To isolate the effect of scheduling from implementation-specific constraints, we fix $B\geq2$ decode slots and define \emph{continuous refill} as an idealized schedule in which every completed response is replaced before the next decode step, while unfinished responses persist across collection boundaries. Consequently, $B_t=B$ at every decode step. We compare this schedule with synchronous rollout, which starts each wave with $B$ responses and waits until all responses finish.

\begin{proposition}[Scheduling-induced KV concentration]
\label{prop:kv-concentration}
Let $B\geq2$ and let $L_1,\ldots,L_B\overset{\mathrm{i.i.d.}}{\sim}\mathcal D$ be positive integer-valued response lengths satisfying $\mathbb E[L^2]<\infty$. Define $M_B=\max_{1\leq i\leq B}L_i$. For a fixed prompt length $P$, the long-run decode-step averages under synchronous rollout and continuous refill satisfy
\begin{equation}
    \overline\ell_{\mathrm{async}}
    =
    P+\frac{\mathbb E[L(L-1)]}{2\mathbb E[L]},
    \qquad
    \mathcal C_{\mathrm{KV}}
    \coloneqq
    \frac{\overline V_{\mathrm{async}}}
         {\overline V_{\mathrm{sync}}}
    =
    \frac{\mathbb E[M_B]}{\mathbb E[L]}
    \geq 1.
    \label{eq:kv-concentration}
\end{equation}
Moreover, $\mathcal C_{\mathrm{KV}}>1$ whenever $\mathcal D$ is non-degenerate.
\end{proposition}

\begin{figure}[t]
  \centering
  \includegraphics[width=0.95\linewidth]{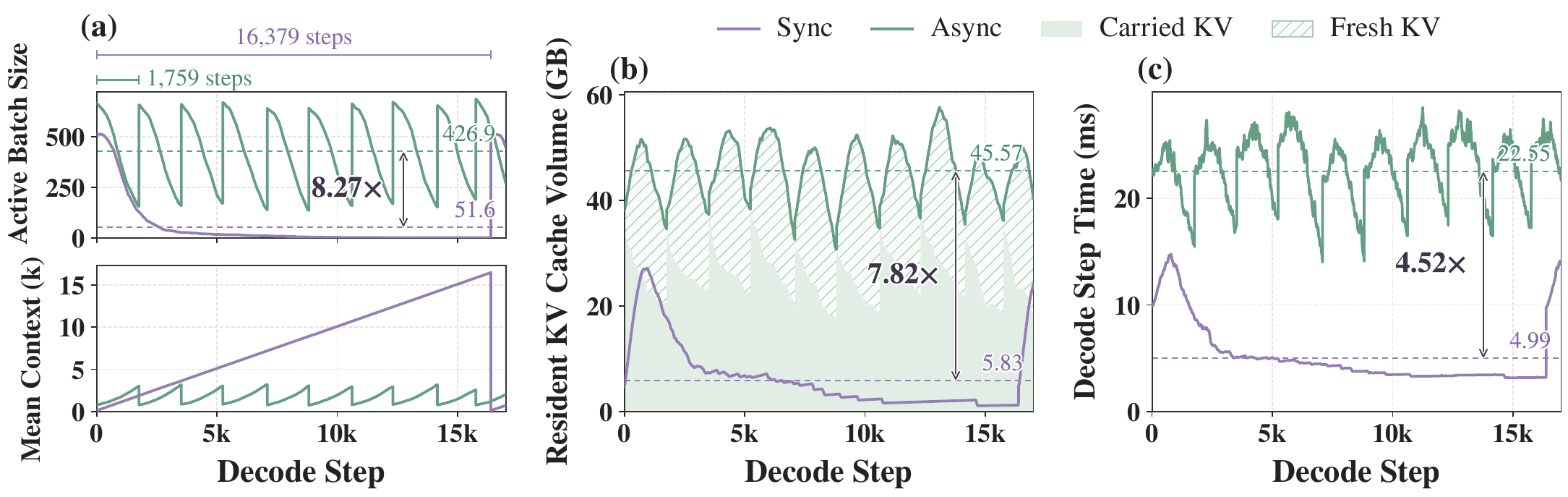}
  \vspace{-5pt}
  \caption{Synchronous and asynchronous rollout traces for Qwen3-4B-Base at mid-training (iteration 300) on H200 (DP2$\times$TP2), with matched KV budgets and concurrency limits. (a) Active batch size and mean context length. (b) Nominal resident KV volume per GPU. (c) Decode-step latency. Brackets identify individual collection intervals in the displayed replica, while dashed horizontal lines denote means computed from unsmoothed decode calls across both replicas over twelve measured rounds. Nominal resident volume measures active logical contexts rather than reserved memory or physical HBM traffic. Appendix~\ref{app:trace-protocol} specifies the trace and timing protocol.}
  \vspace{-10pt}
  \label{fig:concentration}
\end{figure}

Proposition~\ref{prop:kv-concentration} establishes that KV concentration follows directly from the scheduling structure rather than from a particular empirical workload. The first identity captures the length bias induced by observing active responses at decode steps. Longer responses remain active for more decode steps and are therefore overrepresented in the observed batch, an instance of inspection-paradox length bias~\citep{ross1996stochastic}. Indeed, writing $\mathbb E[L]=\mu$ and $\operatorname{Var}(L)=\sigma^2$ yields $\overline\ell_{\mathrm{async}}=P+(\mu-1)/2+\sigma^2/(2\mu)$, making explicit how response-length variation increases the mean context observed in an occupied slot.

The second identity quantifies this occupancy effect through the KV concentration factor $\mathcal C_{\mathrm{KV}}$. A synchronous wave lasts for $M_B$ decode steps, even though its active batch shrinks as shorter responses finish. Continuous refill immediately reuses these vacated slots and thereby eliminates the low-occupancy tail. Since both schedules perform the same expected logical KV work per completed response, continuous refill yields a larger step-averaged resident KV volume by a factor of $\mathbb E[M_B]/\mathbb E[L]$. For any non-degenerate response-length distribution, $\mathbb E[M_B]>\mathbb E[L]$, and hence $\mathcal C_{\mathrm{KV}}>1$. Moreover, $\mathbb E[M_B]$ is nondecreasing in $B$, so within this idealized model, the concentration factor cannot decrease as the number of decode slots grows. If response lengths are deterministic, then $M_B=L$ almost surely and $\mathcal C_{\mathrm{KV}}=1$.

The proposition is not intended to predict the exact KV volume of a practical rollout system. Its role is to show that, even with the response-length distribution and total logical attention work held fixed, asynchronous refill structurally increases the workload presented to each decode step. Practical rollout additionally involves finite request pools, KV page limits, variable prompt lengths, and quota-based collection, so the magnitude of KV concentration must be measured empirically. Appendix~\ref{app:kv-concentration} provides the proof, occupied-slot accounting, and extensions beyond the idealized setting.

\paragraph{KV concentration in practical rollout traces.}
Figure~\ref{fig:concentration} shows how the two components of KV concentration appear in an actual Qwen3-4B-Base rollout. In panel~(a), the synchronous batch steadily loses active responses and eventually enters a long tail dominated by a small number of unusually long generations. Asynchronous rollout instead admits new work into available slots and carries unfinished responses across collection boundaries. Across all measured decode calls, it maintains an $8.27\times$ larger mean active batch and requires $8.35\times$ fewer decode calls per collection interval.

Sustaining this wider batch changes the cost of each remaining decode step. As shown in panel~(b), synchronous rollout reaches its longest contexts only after most responses have completed, whereas asynchronous rollout processes substantial contexts across many responses simultaneously. Its mean nominal resident KV volume is consequently $7.82\times$ larger, increasing mean decode-step latency by $4.52\times$ in panel~(c). The reduction in sequential steps still produces $1.83\times$ higher decode-service throughput in this trace, excluding prefill and collection work. However, the benefit weakens once the concentrated workload approaches system capacity. At the late 8B checkpoint, higher per-step cost and a fixed page budget reduce asynchronous decode-service throughput to $0.98\times$ that of synchronous rollout (Appendix~\ref{app:phase-results}). Thus, asynchronous scheduling removes the utilization tail only to expose a new scaling constraint: the storage and processing cost of the concentrated KV workload.

\subsection{Why Quantize the KV Cache?}
\label{sec:3.2}

\paragraph{Sustained concurrency shifts the workload toward KV processing.}
Under a wide active batch, model-weight traffic is amortized across more responses, while attention must still process the retained context of every response. For a fixed model and numerical precision, the logical memory demand per generated token scales approximately as $W_{\mathrm{bytes}}/B+\kappa\bar\ell$, separating amortized weight reads from context-dependent KV accesses. Increasing concurrency reduces the relative contribution of weight traffic, whereas longer responses continue to increase KV traffic. KV processing therefore becomes increasingly important precisely in the high-concurrency, long-generation regime created by asynchronous rollout.

\begin{figure}[t]
  \centering
  \includegraphics[width=0.95\linewidth]{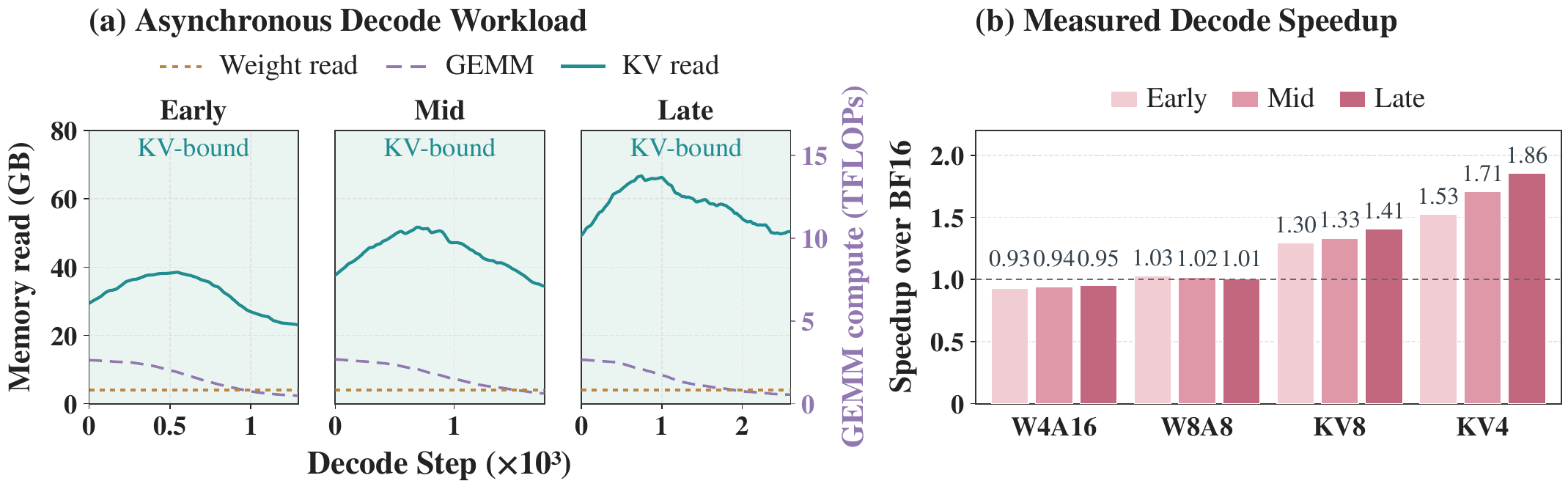}
  \vspace{-5pt}
  \caption{Workload characteristics and quantized decode performance for asynchronous rollout at early, middle, and late stages of Qwen3-4B-Base training (iterations 100, 300, and 500) on H200 with TP2. (a) Nominal per-GPU memory and compute demands, with the axes normalized using peak HBM bandwidth and dense BF16 compute throughput so that equal heights represent equal idealized service times. (b) Decode-forward speedup over BF16 on five fixed, uniform-context workload shapes from each checkpoint. W8A8 denotes FP8 weight-activation execution in this panel. Appendix~\ref{app:decode-replay} specifies the workload construction, numerical implementations, and baseline-time weighting.}
  \label{fig:rollout-quantization}
  \vspace{-10pt}
\end{figure}

Figure~\ref{fig:rollout-quantization}(a) shows how the asynchronous rollout workload evolves across early, middle, and late checkpoints of the same Qwen3-4B-Base training run. KV reads constitute the largest idealized service demand throughout the measured traces and grow as the generated responses become longer. In contrast, estimated weight traffic and GEMM work remain at broadly similar levels across checkpoints.

\paragraph{KV quantization provides the largest measured acceleration.}
Since idealized service demands do not directly determine realized latency, Fig.~\ref{fig:rollout-quantization}(b) replays representative asynchronous decode workloads with W4A16, W8A8, KV8, and KV4 execution. W4A16 uses vLLM Machete, while W8A8 uses an adaptation of vLLM's CUTLASS FP8 GEMM. KV8 and KV4 use our integer paged-attention kernels with BF16 linear layers. Every configuration processes the same decode states, allowing the comparison to isolate numerical execution rather than scheduling differences. The reported timings include the captured forward pass, required numerical conversions, tensor-parallel communication, and captured argmax, while excluding prefill, scheduling, and stochastic sampling outside the graph.

Among the weight-only, weight-activation, and KV cache configurations evaluated here, KV quantization provides the largest speedups. KV4 accelerates decode-forward execution by $1.53\times$, $1.71\times$, and $1.86\times$ at the early, middle, and late checkpoints, respectively, while KV8 achieves $1.30-1.41\times$. The W4A16 and W8A8 implementations remain close to or below BF16 speed on the same workloads. The 8B experiments exhibit the same pattern, with KV4 providing $1.44-1.97\times$ speedup across checkpoints (Appendix~\ref{app:8b-quantization}). These measurements depend on both storage precision and the corresponding numerical and kernel implementation, but consistently identify KV compression as the most effective target in the asynchronous rollout regime.

Fixed-workload replay isolates the reduction in execution cost for a given decode state. During end-to-end rollout, a smaller cache provides an additional benefit by allowing more contexts to fit within the same HBM budget, preserving a wider active batch when KV allocation would otherwise limit request admission. KV cache quantization can therefore accelerate asynchronous rollout through two reinforcing effects: faster processing of the concentrated workload and higher feasible concurrency. The remaining question is whether an aggressive precision such as INT4 can provide these benefits without introducing enough sampling error to alter policy learning.
\section{KV4-RL: Faster Asynchronous Rollout with INT4 KV Caches}
\label{sec:4}

Section~\ref{sec:3.2} identifies KV cache quantization as the most effective acceleration target for asynchronous rollout and shows that INT4 provides the largest execution gains among the evaluated configurations. This finding motivates a seemingly direct approach: storing and processing the rollout KV cache in INT4. In RL rollout, however, the generated responses are not merely inference outputs. They constitute the samples from which policy updates are constructed, so numerical errors that appear modest at an individual prefix can accumulate along trajectories and ultimately alter the learned policy.

The resulting design challenge is to retain the execution and capacity benefits of INT4 while keeping the rollout sampler sufficiently close to the BF16 policy. KV4-RL addresses this tension through three coordinated components. It uses calibration-free online quantization to follow the evolving policy, selectively protects a bounded number of sensitive context positions, and executes the INT4 and protected regions without reconstructing the full cache.

\subsection{A Practical INT4 KV Rollout Path}
\label{sec:4.1}

A quantizer calibrated at a fixed checkpoint can become stale as policy updates shift activation and KV statistics. Repeating offline calibration throughout training would introduce additional synchronization and data-collection overhead. At the same time, KV compression reduces decode cost by lowering cache traffic and increases feasible concurrency by reducing cache storage. Materializing the cache in higher precision before attention would add memory traffic and computation, substantially weakening the decode-cost benefit of compression. We therefore use a calibration-free representation whose quantization parameters are computed online for each newly produced KV vector and whose compressed values are consumed directly by the attention kernel.

We quantize K and V independently using one affine INT4 group for each 128-dimensional token-head vector. The scale $\Delta$ and offset $m$ are stored in FP16 alongside the packed 4-bit codes. For a quantized key $\hat k=\Delta c+m\mathbf1$, the attention score can be written as
\begin{equation}
    q^\top\hat k
    =
    \Delta(q^\top c)
    +
    m\sum_j q_j.
    \label{eq:affine-key}
\end{equation}
Sharing one parameter pair across the head allows the scale and offset to be applied after the dot-product reduction. Finer channel groups may reduce quantization error, but they require more metadata and separately scaled partial reductions. With a head dimension of 128, the resulting representation requires $4+32/128=4.25$ bits per element for the packed codes and quantizer metadata, excluding paging, alignment, and protected side-cache storage.

Following rotation-based KV quantization~\citep{saw}, we apply a fixed normalized Hadamard transform to keys after RoPE and the same transform to queries before attention. Because the transform is orthogonal, it preserves the unquantized QK scores while distributing outlier coordinates before key quantization. Keys are stored in the rotated basis, whereas values remain in their original basis. Combined with the head-wise affine INT4 encoding described above, this yields the calibration-free cache representation that we call \emph{KV4}. To evaluate the representation independently of the deployed integer attention path, we conduct controlled cache-representation ablations that reconstruct the quantized K/V tensors to BF16 and perform the remaining attention computation in BF16. Appendix~\ref{app:quantization-design} uses this setup to evaluate rotation and group-size choices for affine INT4 and to compare its fidelity with NVFP4~\citep{nvfp4} and MXFP4~\citep{mxfp4} at comparable storage costs.

\subsection{Reducing Sampler Mismatch with Selective Protection}
\label{sec:4.2}

\paragraph{Why KV4 alone is insufficient.}
The execution benefit of KV4 does not guarantee that its sampled trajectories remain suitable for policy optimization. Figure~\ref{fig:protection}(a) compares Qwen3-8B-Base training with BF16, KV8, and KV4 rollout under the same asynchronous GRPO configuration with TIS. KV8 closely follows the BF16 trajectory, whereas KV4 tracks it during the early stage but fails to match BF16's later improvement, finishing with substantially lower MATH-500 avg@8 than BF16. Since every checkpoint is evaluated through the same BF16 path, this difference reflects the policies learned from the respective rollout distributions rather than evaluation-time quantization. In this setting, TIS alone is therefore insufficient to preserve learning under KV4 rollout.

To isolate numerical mismatch from trajectory divergence, we compare the next-token distributions produced by BF16 and quantized KV execution at identical BF16-generated prefixes from the same policy checkpoint. Given a set of prefixes $\mathcal H$, we define the shared-prefix conditional mismatch as
\begin{equation}
    \mathcal D_{\mathrm{sampler}}
    =
    \frac{1}{|\mathcal H|}
    \sum_{h\in\mathcal H}
    D_{\mathrm{KL}}\!\left(
        p_{\mathrm{BF16}}(\cdot\mid h)
        \,\middle\|\,
        q_{\mathrm{quant}}(\cdot\mid h)
    \right).
    \label{eq:sampler-kl}
\end{equation}
We compute this quantity before temperature scaling or sampling truncation. The metric captures the conditional distortion introduced by the quantized KV path, but not changes in the prefixes visited by the sampler. It therefore provides a controlled numerical diagnostic that complements the learned-policy comparison in panel~(a). Appendix~\ref{app:fidelity-protocol} describes the evaluation protocol.

Improving fidelity throughout the full INT4 cache through finer groups or additional error compensation would introduce overhead that grows with context length. The near-BF16 learning trajectory of KV8 suggests that INT8 provides sufficient fidelity under the tested training recipe. Applying INT8 to the entire cache, however, would sacrifice part of KV4's execution and capacity advantage. This motivates a bounded selective-precision design that retains the full context in INT4 while using INT8 only for a small set of context positions.

\begin{figure}[t]
  \centering
  \includegraphics[width=0.95\linewidth]{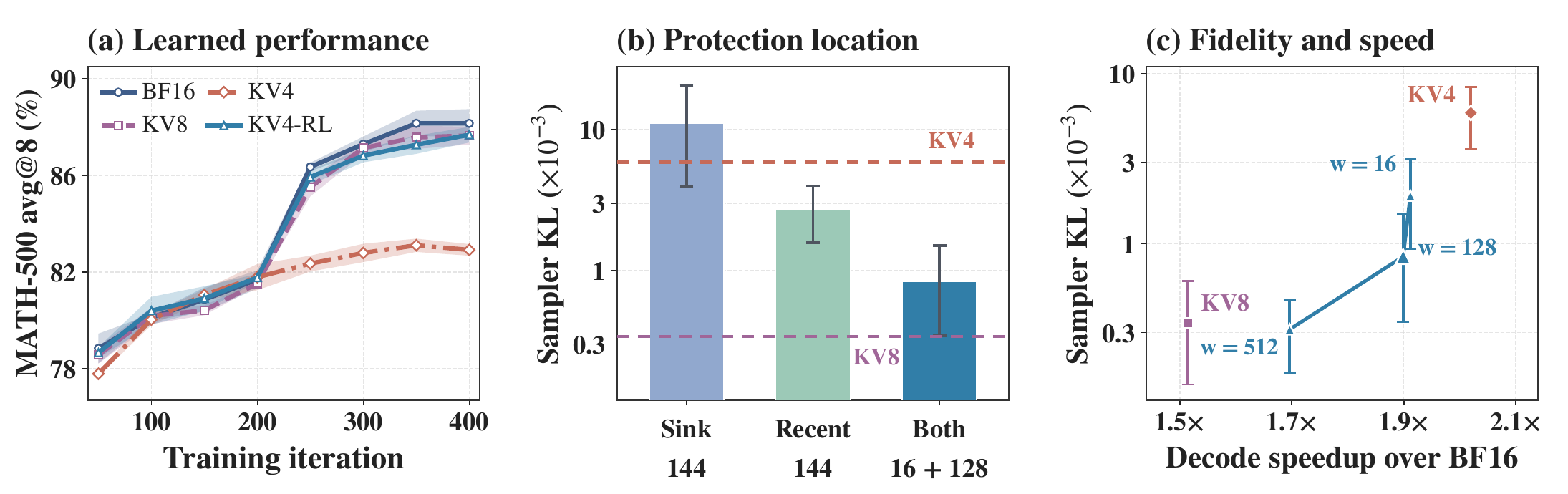}
  \vspace{-5pt}
  \caption{Sampler fidelity, learning, and execution measurements for selective INT8 protection. (a) Qwen3-8B-Base learning curves under BF16 evaluation, with shading denoting standard deviation over three evaluation seeds. (b) Shared-prefix conditional KL for different allocations of a fixed 144-position protection budget, with KV4 and KV8 included as references. (c) Shared-prefix conditional KL and decode-forward speedup over BF16 with sink width $s=16$ and varying recent width $w$. Panels~(b) and~(c) use the mid-training Qwen3-4B-Base checkpoint (iteration 300) and a heterogeneous-state replay suite distinct from Fig.~\ref{fig:rollout-quantization}. KL axes are logarithmic, and error bars show 95\% response-cluster bootstrap intervals. Protected positions use INT8 KV storage with BF16 PV arithmetic.}
  \label{fig:protection}
  \vspace{-10pt}
\end{figure}

\paragraph{Bounded selective protection.}
We refer to the combination of KV4 storage for all positions and bounded INT8 protection as \emph{KV4-RL}. Motivated by attention-sink behavior~\citep{streamingllm,kvsink} and recent-context sensitivity~\citep{kivi,skvq}, KV4-RL maintains INT8 copies of the first $s$ positions in the prompt--response sequence and the most recent $w$ cached positions. Since the protected set contains at most $s+w$ positions per response, its additional precision budget remains bounded independently of context length. With the fixed 144-position budget in Fig.~\ref{fig:protection}(b), allocating 16 positions to the sink and 128 positions to recent context reduces shared-prefix conditional KL by $7.0\times$ relative to KV4 and yields lower mismatch than assigning the full budget to either region alone. These results show that sink and recent protection play complementary roles rather than serving as interchangeable uses of the same precision budget. With the sink width fixed at $s=16$, Fig.~\ref{fig:protection}(c) shows that increasing the recent width reduces sampler mismatch but progressively decreases decode speedup. We select $w=128$ as a practical balance, sharply reducing conditional KL while retaining most of KV4's execution benefit. Additional comparisons of protection width and precision are provided in Appendix~\ref{app:band}.

\paragraph{Fidelity--efficiency trade-off.}
The default configuration achieves $1.90\times$ decode-forward speedup over BF16 and retains approximately $94\%$ of KV4 throughput on the same 4B states. It also restores the late-stage learning improvement absent from the KV4 trajectory in Fig.~\ref{fig:protection}(a), reaching $87.68\%$ MATH-500 avg@8 compared with $88.16\%$ for BF16 rollout under the same TIS correction. Selective protection therefore recovers most of the observed policy-quality loss while retaining most of KV4's execution benefit. Appendix~\ref{app:band} reports the complete fidelity and timing results, and Section~\ref{sec:5} evaluates whether this trade-off carries through to end-to-end acceleration and learned-policy quality across model sizes.

\begin{figure}[t]
  \centering
  \includegraphics[width=\linewidth]{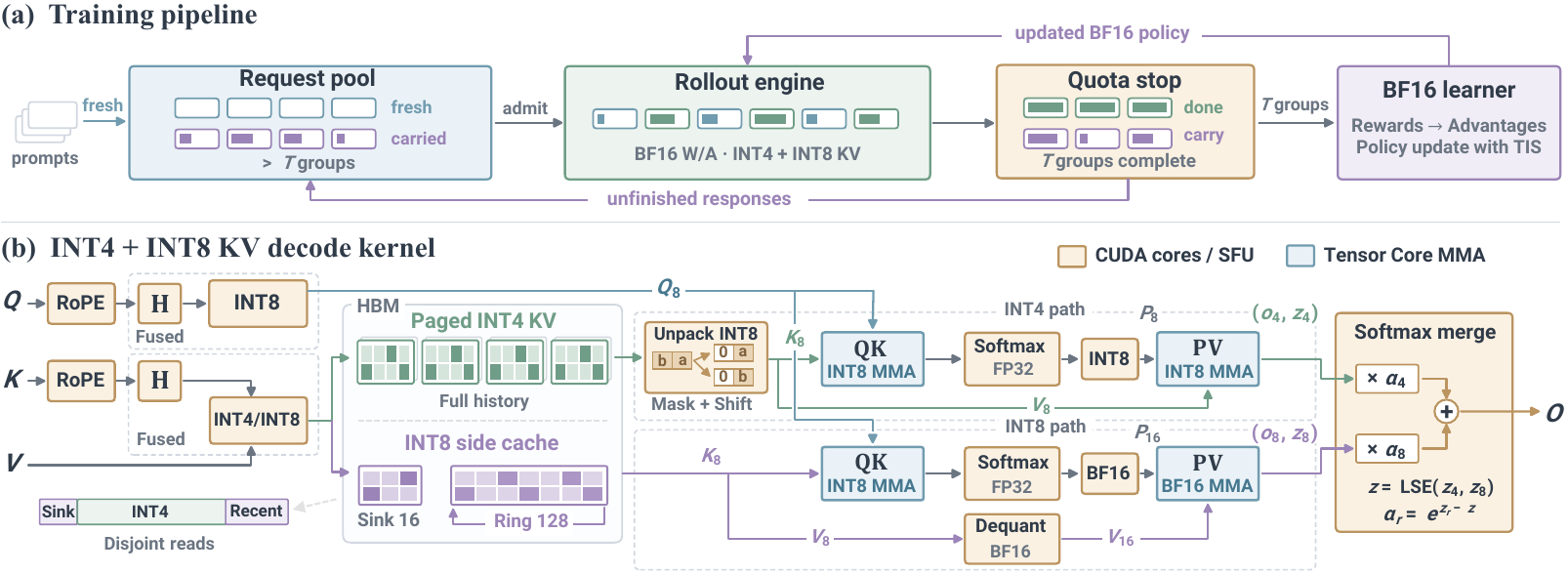}
  \vspace{-5pt}
  \caption{Overview of KV4-RL framework and its attention execution. (a) Quota-based partial rollout transfers completed prompt groups to the BF16 learner and carries unfinished responses across policy updates. (b) The complete history is stored in paged INT4 format, while an INT8 side cache stores the protected sink and recent positions. Attention processes the INT4 and protected regions as disjoint sets and combines their outputs using regional softmax statistics.}
  \label{fig:app-kv4rl-framework}
  \vspace{-10pt}
\end{figure}

\subsection{Efficient Attention Execution}
\label{sec:kernel}

Beyond improving sampler fidelity, the bounded structure of selective protection admits an efficient implementation that preserves the scaling benefits of INT4 KV compression. Unlike finer quantization groups or error-compensation schemes, which introduce additional work for every cached token, its higher-fidelity overhead is restricted to at most $s+w$ positions per response. KV4-RL also maintains the INT4 and protected representations as independent cache structures. The paged INT4 cache stores every position, while a fixed-capacity INT8 side buffer stores additional copies of the sink and recent positions. Both representations are produced directly from the source K/V vectors when each token is generated. Consequently, a position leaving the recent window requires no transfer or requantization into the main cache since its INT4 entry already exists. The evicted side-buffer slot can simply be reused. During attention, protected positions are read from the side buffer and all remaining positions from the paged INT4 cache, forming disjoint regions that avoid double counting. Figure~\ref{fig:app-kv4rl-framework} summarizes this execution design.

A further advantage of this separation is that the two regions can be processed independently and in parallel using numerical paths tailored to their respective roles, and then combined exactly. The main INT4 path quantizes the rotated query to INT8, widens the packed INT4 K/V codes into integer operands, and uses integer QK and PV operations to process the growing context efficiently. The protected path reuses the same prepared INT8 query for QK against the INT8 side keys, while dequantizing the protected values to BF16 for higher-fidelity PV with FP32 accumulation. Since this BF16 PV computation covers at most $s+w$ positions per response, its cost remains bounded. Each path produces a regional attention output and log-normalizer, and the two results are combined through an exact softmax merge. This execution structure allows both caches to remain in their own representations without conversion to a common format. Appendix~\ref{app:kv4-implementation} describes the cache encoding, protected-region allocation, and regional merge in detail.
\section{Experimental Results}
\label{sec:5}

\subsection{Training, Evaluation, and Timing Protocols}
\label{sec:experimental-setup}

\paragraph{Training setup.}
We evaluate KV4-RL on the dense Qwen3-4B-Base and Qwen3-8B-Base models~\citep{qwen3}, using full-parameter GRPO on the English split of \path{open-r1/DAPO-Math-17k-Processed}~\citep{dapo}. Each configuration uses four H200 GPUs, with two TP2 rollout replicas and an FSDP2~\citep{fsdp} learner operating in BF16. Generation and learning alternate on the same GPUs. All rollout precisions share an optimized in-house engine with paged attention, continuous batching, CUDA graphs, and TP communication, while the BF16 baseline uses FlashInfer Tensor Core decode kernels~\citep{flashinfer}. Each iteration collects 128 groups of eight responses with a 16,384-token response cap. TIS clips the token-local correction at $C=5$, and the reference-KL coefficient is $0.001$. We compare KV4-RL against BF16, KV8, and KV4 rollout, training the 4B and 8B models for 600 and 400 updates, respectively. Appendix~\ref{app:engine} describes the common engine optimizations, and Appendix~\ref{app:training-setup} provides the full training configuration, including request-pool and memory-allocation settings.

\paragraph{Learned-policy evaluation.}
All checkpoints are evaluated using BF16 weights and KV caches, so the reported accuracy measures the learned policy independently of the numerical path used during rollout. Evaluation uses temperature $0.6$, top-$p=0.95$, and a 16,384-token response cap. We report avg@$8$ on MATH-500~\citep{math5002024}, OlympiadBench~\citep{olympiadbench2024}, and MinervaMath~\citep{minerva2022}, and avg@$32$ on AIME 2024/2025~\citep{aime2024,aime2025} and AMC 2023~\citep{amc2023}. Here avg@$k$ denotes mean correctness over $k$ sampled answers per problem, distinct from pass@$k$, which records whether at least one answer is correct. Periodic evaluation covers MATH-500 and both AIME sets, while final evaluation covers all six tasks. Each rollout configuration uses one training run, with checkpoints evaluated using three different seeds. The reported standard deviations therefore measure sampling variability at fixed checkpoints. Appendix~\ref{app:evaluation} provides dataset and grading details.

\paragraph{Timing protocol.}
We measure execution time at three nested scopes. Fixed-workload replay measures captured decode forwards at matched states and therefore isolates the execution benefit of KV compression. Rollout collection additionally includes prefill, carried-prefix reconstruction, scheduling, output collection, and reward scoring, while full-iteration time further includes learner computation, parameter synchronization, and phase transitions. End-to-end rollout also captures the capacity benefit of reduced KV storage: under the same memory budget, it can admit more resident context and sustain higher decode concurrency when KV capacity limits scheduling. We exclude initialization, restart downtime, evaluation, and checkpoint saving. Cumulative speedup is computed as the ratio of summed BF16 time to summed method time over the same update interval, rather than as the mean of per-iteration speedup ratios. Since each training configuration generates responses from its own evolving policy, collection and iteration times reflect execution efficiency, feasible concurrency, and the workload produced during learning. Appendix~\ref{app:timing-boundaries} specifies the timing boundaries.

\subsection{End-to-End Acceleration}
\label{sec:5.1}

\begin{figure}[t]
  \centering
  \includegraphics[width=0.95\linewidth]{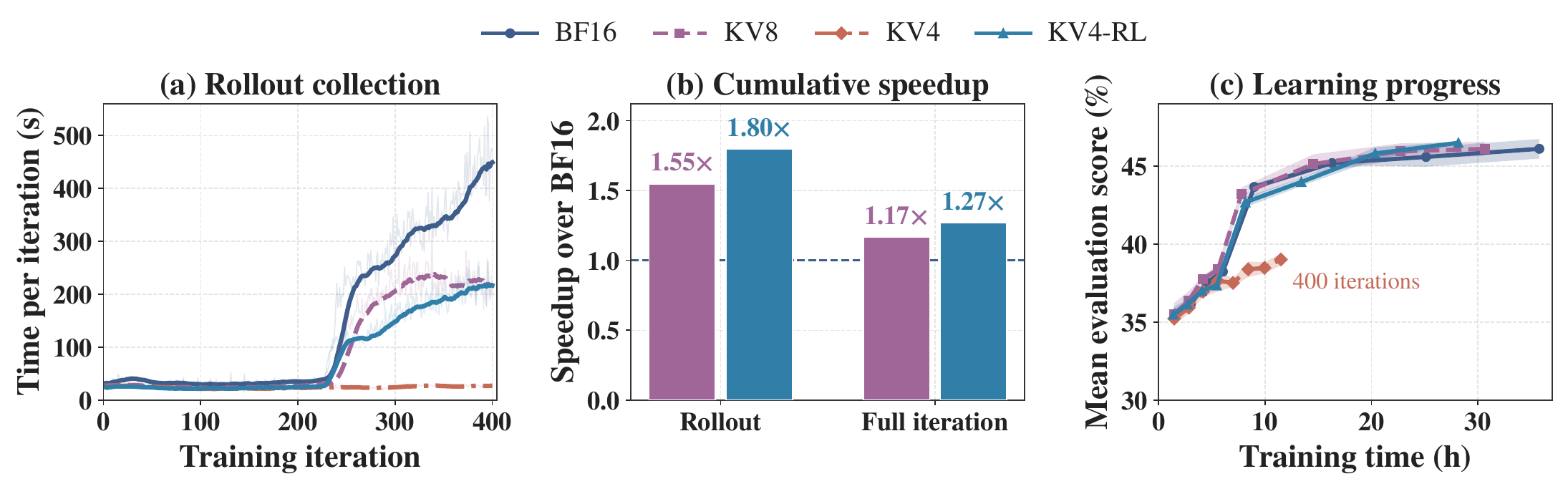}
  \vspace{-5pt}
  \caption{End-to-end acceleration and learning progress for Qwen3-8B-Base over 400 training iterations. (a) Rollout collection time per iteration, with 20-iteration means over faint raw traces. (b) Ratios of cumulative BF16 time to cumulative method time for rollout collection and full iterations. KV4 is omitted because its generated responses are substantially shorter. (c) Equally weighted mean accuracy over MATH-500 and AIME 2024/2025 versus cumulative full-iteration time, with error bars denoting standard deviation over three evaluation seeds. Cumulative time excludes initialization, restart downtime, evaluation, and checkpoint saving.}
  \label{fig:e2e-acceleration}
  \vspace{-10pt}
\end{figure}

\textbf{KV4-RL substantially accelerates rollout and end-to-end training.}~
Across 400 Qwen3-8B-Base updates, KV4-RL achieves $1.80\times$ cumulative rollout speedup and $1.27\times$ cumulative full-iteration speedup over BF16, as shown in Fig.~\ref{fig:e2e-acceleration}. It reaches comparable periodic evaluation accuracy in less cumulative training time, while its run-wide mean response length remains within $1\%$ of BF16 (Fig.~\ref{fig:training-trajectory}(b)). The timing improvement therefore accompanies a closely matched generation workload. Unlike the fixed-workload replay in Section~\ref{sec:3.2}, this end-to-end gain includes both faster processing of the resident KV workload and the additional decode concurrency enabled by fitting more context within the same memory budget. The smaller full-iteration gain reflects the remaining learner-side cost: rollout collection accounts for approximately $47\%$ of cumulative BF16 iteration time, while KV4-RL leaves the BF16 learner computation unchanged. After rollout is accelerated, learner computation accounts for $60.5\%$ of KV4-RL iteration time over the final 50 updates (Appendix~\ref{app:iteration-breakdown}), limiting the fraction of total iteration time that further rollout optimization can reduce.

\textbf{Rollout acceleration grows with response length.}~
KV4-RL provides smaller but consistent gains on Qwen3-4B-Base, where responses remain shorter. Across 600 updates, it achieves $1.57\times$ cumulative rollout speedup and $1.17\times$ cumulative full-iteration speedup (Appendix~\ref{app:4b-continuations}). The run-wide mean response lengths are closely matched at 1.45k tokens for BF16 and 1.46k for KV4-RL, indicating that the timing difference is not explained by a smaller generated workload. The contrast becomes more pronounced late in training: over the final 50 iterations, both configurations average approximately 2.0k tokens per response at 4B, compared with 5.5k at 8B. Consistently, within the 8B experiment, rollout speedup rises from $1.45\times$ over iterations 1--100 to $1.97\times$ over iterations 351--400 as responses become longer. These trends align with the workload analysis in Section~\ref{sec:3.2}. Longer contexts increase both the KV traffic processed at each decode step and pressure on the finite cache pool, strengthening the execution and capacity benefits of KV compression.

\subsection{Final Accuracy and Learning Trajectories}
\label{sec:5.2}

\begin{table}[t]
  \centering
  \caption{
  Final Qwen3-Base benchmark accuracy (\%) under BF16 evaluation.
  The 4B and 8B checkpoints are evaluated after 600 and 400 training iterations, respectively.
  MATH-500, OlympiadBench, and MinervaMath use avg@$8$, while AIME 2024/2025 and AMC 2023 use avg@$32$.
  Subscripts denote standard deviation over three evaluation seeds.
  Mean assigns equal weight to the six tasks within each seed before computing the reported mean and standard deviation.
  All values are rounded to one decimal place.
  }
  \label{tab:final-performance}
  \vspace{3pt}
\begingroup
\small
\setlength{\tabcolsep}{3pt}
\renewcommand{\arraystretch}{1.0}
\resizebox{0.95\linewidth}{!}{%
\begin{tabular}{@{}lcccccccc@{}}
\toprule
& & \multicolumn{3}{c}{avg@$8$} & \multicolumn{3}{c}{avg@$32$} & \\
\cmidrule(lr){3-5}\cmidrule(lr){6-8}
\textbf{Model} & \textbf{Rollout} & \textbf{MATH-500} & \textbf{OlympiadBench} & \textbf{MinervaMath} & \textbf{AIME24} & \textbf{AIME25} & \textbf{AMC23} & \textbf{Mean} \\
\midrule
\multirow{4}{*}{Qwen3-8B-Base} & BF16 & $88.2_{\pm 0.6}$ & $56.3_{\pm 0.2}$ & $51.9_{\pm 0.4}$ & $29.5_{\pm 1.9}$ & $20.7_{\pm 0.4}$ & $67.9_{\pm 1.2}$ & $52.4_{\pm 0.5}$ \\
 & KV8 & $87.6_{\pm 0.4}$ & $56.6_{\pm 0.4}$ & $52.0_{\pm 0.6}$ & $29.4_{\pm 0.6}$ & $21.3_{\pm 0.5}$ & $69.9_{\pm 1.1}$ & $52.8_{\pm 0.4}$ \\
 & KV4 & $82.9_{\pm 0.2}$ & $50.1_{\pm 0.2}$ & $51.2_{\pm 0.8}$ & $18.5_{\pm 0.7}$ & $15.6_{\pm 0.3}$ & $68.4_{\pm 1.4}$ & $47.8_{\pm 0.1}$ \\
\noalign{\vskip 1.5pt}
\cdashline{2-9}[4pt/2pt]
\noalign{\vskip 1.5pt}
 & KV4-RL & $87.7_{\pm 0.3}$ & $56.8_{\pm 0.4}$ & $51.8_{\pm 1.4}$ & $29.7_{\pm 0.6}$ & $22.1_{\pm 0.7}$ & $68.9_{\pm 1.1}$ & $52.8_{\pm 0.1}$ \\
\midrule
\multirow{4}{*}{Qwen3-4B-Base} & BF16 & $83.0_{\pm 0.2}$ & $50.1_{\pm 0.2}$ & $48.0_{\pm 0.7}$ & $19.1_{\pm 0.2}$ & $16.5_{\pm 0.4}$ & $65.2_{\pm 1.0}$ & $47.0_{\pm 0.1}$ \\
 & KV8 & $82.5_{\pm 0.1}$ & $49.2_{\pm 0.5}$ & $46.9_{\pm 0.5}$ & $18.2_{\pm 0.6}$ & $16.8_{\pm 1.0}$ & $66.2_{\pm 0.3}$ & $46.6_{\pm 0.3}$ \\
 & KV4 & $81.7_{\pm 0.5}$ & $49.1_{\pm 0.5}$ & $48.0_{\pm 0.7}$ & $16.8_{\pm 0.4}$ & $14.1_{\pm 0.4}$ & $67.6_{\pm 1.3}$ & $46.2_{\pm 0.0}$ \\
\noalign{\vskip 1.5pt}
\cdashline{2-9}[4pt/2pt]
\noalign{\vskip 1.5pt}
 & KV4-RL & $82.4_{\pm 0.3}$ & $48.8_{\pm 0.0}$ & $49.1_{\pm 0.7}$ & $18.2_{\pm 0.9}$ & $13.6_{\pm 0.1}$ & $67.5_{\pm 0.7}$ & $46.6_{\pm 0.3}$ \\
\bottomrule
\end{tabular}
}
\endgroup

\end{table}

\textbf{KV4-RL maintains competitive final-policy accuracy.}~
Table~\ref{tab:final-performance} shows that KV4-RL remains within 0.5 percentage points of BF16 in the equally weighted six-task mean at both model sizes. At 8B, KV4-RL achieves $52.8\%$ mean accuracy compared with $52.4\%$ for BF16, while the corresponding values at 4B are $46.6\%$ and $47.0\%$. Task-level outcomes vary around these aggregate results. At 4B, for example, AIME 2025 accuracy decreases from $16.5\%$ to $13.6\%$, while improvements on other tasks partly offset this drop. Additional learning curves on MATH-500 and both AIME benchmarks, together with endpoint pass@$k$ results, are reported in Appendix~\ref{app:protection-learning}.

In contrast, KV4 exhibits a clearer aggregate degradation, falling 4.6 percentage points below BF16 at 8B and 0.8 points below BF16 at 4B. The larger 8B deficit occurs in the longer-context regime, although this comparison does not isolate the effects of response length and model size. KV4-RL closes this gap at 8B while retaining most of KV4's execution benefit. Together with the controlled fidelity and learning results in Section~\ref{sec:4.2}, this comparison supports selective protection as the component that enables INT4 KV rollout to preserve learned-policy quality under repeated policy optimization.

\begin{figure}[t]
  \centering
  \includegraphics[width=0.95\linewidth]{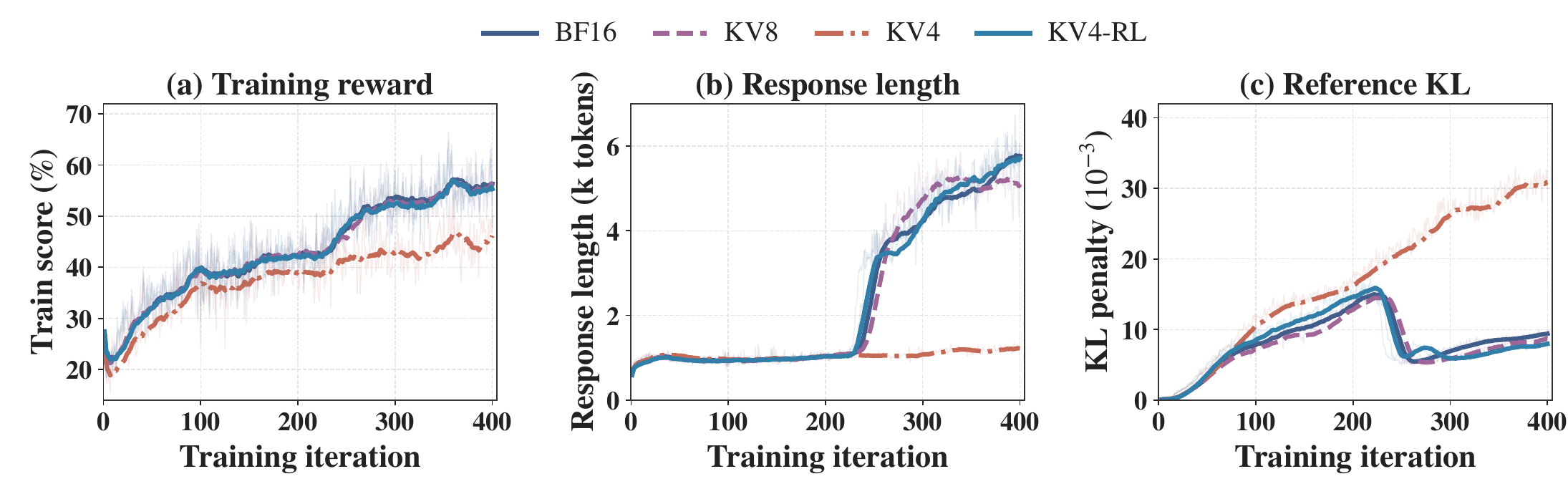}
  \vspace{-5pt}
  \caption{Training trajectories for Qwen3-8B-Base. (a) Mean correctness reward. (b) Mean completed-response length. (c) Logged reference-KL penalty computed with the $k_3$ estimator against the initial policy. This statistic is distinct from the shared-prefix sampler KL in Eq.~\ref{eq:sampler-kl}. Thick curves show 20-iteration means, and faint curves show the corresponding raw traces.}
  \label{fig:training-trajectory}
  \vspace{-10pt}
\end{figure}

\textbf{KV4-RL preserves the learning and generation trajectory of BF16 rollout.}~
Figure~\ref{fig:training-trajectory} shows how the generated workload evolves alongside learning. KV4-RL follows BF16 through the transition to longer responses and higher training rewards, while maintaining a similar level of logged reference-KL penalty. KV4 also improves its training reward, but exhibits less late-stage response-length growth and a continuing increase in the reference-KL statistic. This statistic measures divergence from the fixed initial policy, whereas Eq.~\ref{eq:sampler-kl} measures conditional numerical mismatch between execution paths at shared prefixes. KV4's shorter responses reduce the amount of generated work as well as collection time, so its raw timing does not provide a comparable-workload acceleration result and is omitted from Fig.~\ref{fig:e2e-acceleration}(b). KV4-RL instead closely tracks the BF16 response-length and learning trajectories, providing a better-controlled workload for the cumulative timing comparison.

\section{Discussion}
\label{sec:6}

\subsection{KV Quantization as a Scaling Primitive for Long-Horizon RLVR}

\paragraph{From sustained concurrency to KV-aware scaling.}
Asynchronous and partial rollout improve utilization by sustaining generation across long and heterogeneous responses~\citep{areal,rollflash,kimi,april}. Our central finding is that this concurrency also concentrates context processing into fewer and more expensive decode steps, shifting the rollout bottleneck toward the KV cache. KV quantization directly addresses this workload by reducing both the execution cost of each decode state and the memory required to sustain a large active batch. KV4-RL further shows that these gains can translate into faster policy optimization when aggressive compression is paired with sufficient sampler fidelity. KV quantization therefore provides a systems primitive for converting sustained rollout concurrency into end-to-end RLVR acceleration. This mechanism depends on the workload maintained by the rollout engine and is also relevant to disaggregated systems whose rollout workers retain many long prefixes~\citep{staleflow}.

\paragraph{Expanding the scaling frontier.}
Recent frontier architectures substantially reduce context-dependent memory and computation through compressed, local, or hybrid attention designs~\citep{gemma3,minimax01,kimilinear,deepseekinsights}. These advances expand the feasible context and concurrency range, but the released capacity is increasingly used to support longer reasoning, larger active batches, and more persistent generation state. The same pressure therefore reappears as systems approach the new operating frontier. MoE architectures alter the balance of FFN computation and communication, but their KV workload remains governed by the attention structure~\citep{qwen3,kimik2,deepep}. This consideration becomes particularly important for frontier reasoning and agentic RL, where rollouts extend into long multi-turn trajectories with tool outputs, environment observations, and persistent histories~\citep{deepseekr1,kimi,kimik2,earl2026,agentcontext2026}. KV quantization complements these architectural advances by pushing the feasible context and concurrency frontier further, positioning it as an increasingly important foundation for scaling long-horizon RLVR.

\subsection{Limitations}

While encouraging, our evaluation focuses on two dense Qwen3 models and mathematical reasoning workloads under a colocated asynchronous GRPO configuration. This setting directly captures the long-generation, high-concurrency regime in which KV concentration becomes important, but it does not cover the broader diversity of modern RLVR systems. Although the underlying mechanism applies whenever rollout workers sustain many long resident contexts, its quantitative impact may vary with model architecture, attention structure, scheduler design, and system placement. Further validation on MoE and hybrid-attention models, disaggregated rollout systems, and broader long-horizon workloads such as agentic and multi-turn RLVR remains an important direction.
\section{Conclusion}
\label{sec:conclusion}

In this work, we identified KV concentration as a defining systems consequence of asynchronous RLVR rollout: sustained concurrency reduces low-utilization tail generation but concentrates context processing into fewer and more expensive decode steps. Our analysis and matched-workload measurements showed that this workload shift makes KV cache quantization the most effective acceleration target among the tested quantization configurations. We further found that INT4 KV rollout can degrade learning despite token-level importance correction, demonstrating that aggressive cache compression must be accompanied by sufficient sampler fidelity. Guided by these findings, we proposed KV4-RL, which combines calibration-free INT4 KV storage with bounded INT8 protection of attention sinks and recent context. Across full-parameter GRPO experiments, KV4-RL preserves learned-policy quality close to BF16-rollout training while substantially accelerating rollout and end-to-end training, with larger gains as responses become longer. Together, these results establish KV-aware precision allocation as a practical path for converting sustained rollout concurrency into training acceleration and position KV quantization as an important scaling primitive for long-horizon RLVR.

\FloatBarrier
\bibliographystyle{plainnat}
\bibliography{paper}

\clearpage
\appendix
\counterwithout{equation}{section}

\setcounter{figure}{0}
\setcounter{table}{0}
\setcounter{equation}{0}
\renewcommand{\thefigure}{A\arabic{figure}}
\renewcommand{\thetable}{A\arabic{table}}
\renewcommand{\theequation}{A\arabic{equation}}
\ifdefined\theHfigure\renewcommand{\theHfigure}{appendix.\arabic{figure}}\fi
\ifdefined\theHtable\renewcommand{\theHtable}{appendix.\arabic{table}}\fi
\ifdefined\theHequation\renewcommand{\theHequation}{appendix.\arabic{equation}}\fi

\section{Implementation Details}
\label{app:implementation}
\label{app:method}

All experiments use an in-house RLVR framework with a tensor-parallel rollout engine and an FSDP2 learner. The implementation incorporates the core inference and training optimizations used in vLLM and verl~\citep{vllm,verl}, including direct reuse of vLLM's communication primitives and FlashInfer attention~\citep{flashinfer}. These optimizations are enabled in the BF16 baseline as well as the quantized configurations. Precision changes are therefore evaluated within a common optimized execution stack.

\paragraph{Why a separate framework.}
Combining quota-based partial rollout with our paged INT4 attention requires coordinated control over response lifetime and cache execution. Prefixes carried across a policy update must be resumed with priority and used to rebuild both the packed KV history and its protected side cache. The custom cache representation also affects memory-budget accounting and attention dispatch within captured model forwards. Supporting this combination through verl and vLLM would require changes across both the collection layer and the engine's cache interfaces. We implement these components together to make their interaction explicit while retaining the serving and learning optimizations detailed below.

\subsection{Collection and Policy Updates}
\label{app:collection}

Fig.~\ref{fig:app-async-rollout} in Section~\ref{sec:preliminaries} illustrates how collection proceeds across policy updates. Synchronous collection finishes every response assigned to the current iteration. Its request queue eventually empties and the remaining long responses continue with progressively fewer companions. Partial collection instead maintains an oversubscribed pool and stops after enough complete prompt groups have become available. The collector selects the oldest $T$ completed groups for training and retains any additional completed groups for later selection. Unfinished responses preserve their token prefixes and resume with priority when the next collection begins.

In the evaluated configuration, the global quota is $T=128$ groups with $G=8$ responses per group. An oversubscription factor of $\beta=1$ maintains up to $T(1+\beta)$ groups across the rollout replicas. A group that reaches the age limit $\eta=12$ must complete before collection can stop. This requirement can extend collection beyond the completion quota. Here asynchronous rollout refers to the lifetime of responses across updates. Generation and optimization alternate on the same four GPUs.

Each response retains the behavior log-probability recorded when each token was sampled. These probabilities remain attached to the trajectory when it crosses an update or waits in the completed-group pool. The learner evaluates the current pre-update BF16 policy on the collected tokens and forms the TIS correction in Eq.~\ref{eq:prelim-tis}. The rollout configuration uses temperature one without top-$k$ or nucleus truncation, so the recorded softmax probabilities correspond to its sampling distribution. After optimization, the updated weights are copied into the rollout engine. Re-prefill then rebuilds the carried prefixes under those weights while preserving the sampled tokens and their historical behavior probabilities.

\subsection{Common Execution Optimizations}
\label{app:engine}

\paragraph{Scheduling and attention.}
The rollout engine stores KV entries in 16-token pages and admits requests according to available pages and the concurrency limit. Continuous batching reuses slots released by completed responses while queued requests remain. Chunked prefill shares a token budget with decode requests, and active-row compaction keeps completed requests out of subsequent model forwards. The BF16 baseline uses FlashInfer paged attention with its Tensor Core decode path~\citep{flashinfer}. These choices provide the paged storage and dynamic batching mechanisms used in modern serving engines such as vLLM~\citep{vllm}.

\paragraph{Model execution and communication.}
The engine fuses QKV projections and the MLP gate/up projections. Compiled model execution reduces dispatch overhead and intermediate tensor allocations around attention. Multiple CUDA graphs cover different decode batch sizes, allowing the scheduler to select a captured forward as concurrency changes. Tensor-parallel communication reuses vLLM's custom all-reduce implementation with PyNCCL as a fallback. BF16 and quantized samplers share these scheduling and execution mechanisms. Precision-specific conversion work remains inside the timed forward when it is required during decoding.

\paragraph{Learner execution.}
The learner updates all model parameters with FSDP2. BF16 forward and backward computation use FP32 master parameters and optimizer states. Padding is removed before variable-length attention, and token-budget microbatches balance the work of responses with different lengths. Fused projection and log-probability computation avoid retaining full-vocabulary logits for every training token. Activation checkpointing reduces the update's peak memory. The FSDP2 state-management and offloading utilities follow verl's implementation, while the rollout engine and phase controller support the combined partial-collection and cache-precision design.

\paragraph{Phase memory management.}
Before rollout, optimizer-only state and the frozen reference model are offloaded to make memory available for the KV pool. The updated learner parameters are streamed into the engine's tensor-parallel weight buffers. Before optimization, rollout allocations are released and the learner state is restored. Virtual-memory-backed allocations preserve addresses used by captured graphs across these transitions. The memory profiler reserves model state and execution workspace before sizing the KV pool, including precision-dependent metadata and protected side-cache storage. If $M_q$ denotes the memory remaining for cache format $q$ and $b_q$ denotes its bytes per 16-token page, the allocator provisions $\lfloor M_q/b_q\rfloor$ pages. The scheduler admits a request only when both a request slot and its required pages are available. KV compression can therefore improve rollout through two distinct mechanisms: reducing the cost of processing the current decode state and delaying the point at which page capacity restricts active concurrency. The bounded side-cache allocation reduces the available page pool for KV4-RL but does not grow with the retained context length.

\subsection{Paged INT4 Attention with Selective Protection}
\label{app:kv4-implementation}
\label{app:storage}

Section~\ref{sec:kernel} and Fig.~\ref{fig:app-kv4rl-framework} summarize the numerical paths and their integration with partial rollout. The following paragraphs detail the encoding, allocation, and regional merge.

\paragraph{Encoding and allocation.}
Following the K-only rotation design of SAW-INT4~\citep{saw}, Q and K use a normalized Hadamard transform after RoPE, while V remains in its original basis. Each 128-element group uses affine INT4 codes with FP16 scale and offset. Codes occupy half a byte per scalar and metadata adds $4/128$ bytes. INT8 side entries use symmetric per-token/head scaling and are computed from source vectors before INT4 rounding. The INT4 history covers every position, so a token leaving the recent ring needs no historical requantization. Sink size 16 aligns the beginning of the unprotected INT4 region with a page boundary. Side storage is reserved per request slot and included in the complete cache allocation.

\paragraph{Integer decode path.}
Query rotation and INT8 quantization are fused into query preparation. The attention kernel loads packed pages into shared memory and widens four-bit codes using mask and shift operations. Integer QK MMA accumulates into INT32. Scales and affine offsets convert these accumulators to floating-point scores. After FP32 softmax, the main PV path forms an eight-bit probability operand with the required value-scale factors. A second integer MMA and affine correction produce the regional output without materializing the full history in BF16.

\paragraph{Protected path and reconstruction.}
The protected path reuses the prepared query for integer QK and dequantizes only the INT8 side values for BF16 PV with FP32 accumulation. It avoids the main path's probability quantization. Prefill and carried-prefix reconstruction use BF16 attention and populate both cache representations from newly computed K/V. Historical behavior probabilities remain attached to the sampled tokens during this reconstruction, as described in Section~\ref{app:collection}.

\paragraph{Regional softmax merge.}
For each disjoint region $r\in\{4,8\}$, let $o_r$ be its locally normalized output and $z_r$ its log-normalizer. The outputs are merged as
\begin{equation}
 z=\log(e^{z_4}+e^{z_8}),\qquad
 o=e^{z_4-z}o_4+e^{z_8-z}o_8.
 \label{eq:merge-main}
\end{equation}
A stable log-sum-exp reduction evaluates this identity, with zero weight for empty regions. The merge is exact for the regional scores and outputs, which already reflect their numerical paths. The same identity merges context splits within each region.

\section{Experimental Setup and Measurement Protocols}
\label{app:experimental-setup}
\label{app:section34-evidence}

Our experiments separate workload characterization from execution timing and learned-policy evaluation. Frozen-checkpoint traces expose the workload created by scheduling. Decode replay holds that workload fixed when comparing precision configurations. Training and common-sampler evaluation then measure the resulting rollout cost and policy quality. The protocols below specify the population and timing boundary used in each comparison.

\subsection{Models and Training Configuration}
\label{app:training-setup}

We train Qwen3-4B-Base and Qwen3-8B-Base~\citep{qwen3} on the English training split of \path{open-r1/DAPO-Math-17k-Processed}~\citep{dapo}. Each problem is supplied without a chat template. We append the instruction ``Let's think step by step and output the final answer within \texttt{\textbackslash boxed\{\}}.'' Math-Verify assigns a correctness reward by comparing the extracted answer with the reference. The training scorer follows verl's boxed-answer verification procedure. We use full-parameter GRPO with group-normalized advantages and token-mean loss aggregation. Table~\ref{tab:app-training} lists the shared recipe.

\begin{table}[htbp]
  \centering
  \small
  \caption{Shared training configuration with ten learning-rate warmup updates. Frozen trace and replay experiments use explicit KV budgets instead of automatic training-pool sizing.}
  \label{tab:app-training}
  \begin{tabular}{@{}p{0.44\linewidth}p{0.52\linewidth}@{}}
    \toprule
    Setting & Value \\
    \midrule
    Hardware per training run & 4$\times$ NVIDIA H200 GPUs \\
    Rollout parallelism & Two DP replicas with TP2 each \\
    Learner & FSDP2 with BF16 compute and FP32 master weights \\
    Prompt groups / responses per group & $T=128$ / $G=8$ \\
    Prompt / response token limits & 1,024 / 16,384 \\
    Sampling temperature / top-$p$ & $1.0$ / $1.0$ with top-$k$ disabled \\
    Request-pool oversubscription / age limit & $\beta=1$ / $\eta=12$ iterations \\
    Optimizer & AdamW with $(\beta_1,\beta_2)=(0.9,0.999)$ \\
    Learning rate / warmup & $10^{-6}$ / 10 updates then constant \\
    Weight decay / gradient-norm limit & $0.01$ / $1.0$ \\
    PPO epochs / global minibatch & One / all $T$ prompt groups \\
    Policy-ratio clipping / dual clip & $[0.8,1.2]$ / $3.0$ \\
    Reference-KL coefficient / estimator & $0.001$ / $k_3$ \\
    TIS correction & Token-level upper clip at $C=5$ \\
    Gradient checkpointing / padding removal & Enabled / enabled \\
    KV page size / maximum active requests & 16 tokens / 1,024 per replica \\
    KV4-RL protection & Sink 16 and recent 128 in INT8 \\
    \bottomrule
  \end{tabular}
\end{table}

Policy optimization uses the clipped surrogate in Section~\ref{sec:preliminaries} with reference-KL regularization. For negative advantages, dual clipping bounds the unclipped surrogate using coefficient $3.0$. We retain the same optimization settings across cache-precision configurations. Training microbatches use an automatically profiled token budget with padding removed. Precision arms share the same hardware and maximum-request limit. For each arm, the allocator first reserves model state, execution workspace, quantization metadata, and any protected side-cache storage, then converts the remaining HBM capacity into pages of the corresponding KV format. Request admission requires both an available request slot and sufficient free KV pages. Lower-precision caches can therefore retain more context and sustain more active requests before page capacity becomes binding, even though all arms share the same limit of 1,024 active requests per replica.

The workload and quantization-evolution results use one BF16-trained lineage per model. The 4B checkpoints are iterations 100/300/500 and the 8B checkpoints are iterations 100/250/400. Continued 8B runs retain the same training lineage. Each checkpoint comparison holds policy parameters fixed. The trace manifests record PyTorch 2.10.0 with CUDA 12.8. Kernel attribution for the measured precision configurations is given in Appendix~\ref{app:decode-replay}.

\subsection{Policy Evaluation}
\label{app:evaluation}

We evaluate all trained policies with BF16 weights and BF16 KV caches in a common evaluator. This removes the test-time quantized sampler from the comparison of learned policies. MATH-500~\citep{math5002024} is a 500-problem subset of MATH~\citep{math2021}. We also evaluate OlympiadBench~\citep{olympiadbench2024} and MinervaMath~\citep{minerva2022}. The competition sets comprise AIME 2024/2025~\citep{aime2024,aime2025} and AMC 2023~\citep{amc2023}. Table~\ref{tab:app-evaluation} records the exact dataset distributions and sample counts.

Sampled evaluation uses temperature $0.6$ and top-$p=0.95$ with a 16,384-token response cap. Evaluation prompts allow up to 2,048 tokens and use the same boxed-answer instruction. Periodic evaluation covers MATH-500 and both AIME sets. Endpoint evaluation additionally covers the remaining three datasets.

\begin{table}[htbp]
  \centering
  \small
  \caption{Evaluation tasks. Every task is included at the endpoint. OlympiadBench refers to the English text-only open-ended mathematics competition subset. MinervaMath is the OCWCourses university-level STEM set.}
  \label{tab:app-evaluation}
  \begin{tabular}{@{}lrrll@{}}
    \toprule
    Task & Questions & $k$ & Periodic & Dataset source \\
    \midrule
    MATH-500 & 500 & 8 & Yes & \texttt{HuggingFaceH4/MATH-500} \\
    AIME 2024 & 30 & 32 & Yes & \texttt{Maxwell-Jia/AIME\_2024} \\
    AIME 2025 & 30 & 32 & Yes & \texttt{yentinglin/aime\_2025} \\
    AMC 2023 & 40 & 32 & No & \texttt{math-ai/amc23} \\
    OlympiadBench & 674 & 8 & No & \texttt{math-ai/olympiadbench} \\
    MinervaMath & 272 & 8 & No & \texttt{math-ai/minervamath} \\
    \bottomrule
  \end{tabular}
\end{table}

Let $c_i$ count correct answers among the $k$ samples for question $i$. For $n$ questions, we report $\mathrm{avg@}k=(nk)^{-1}\sum_{i=1}^{n}c_i$ and $\mathrm{pass@}k=n^{-1}\sum_{i=1}^{n}\mathbf1\{c_i>0\}$. Thus avg@$k$ averages individual sampled answers and pass@$k$ records whether at least one answer succeeds. An aggregate across the six tasks gives each task equal weight. Error bars across evaluation seeds use the sample standard deviation of per-seed scores at a fixed checkpoint.

The evaluator uses Math-Verify for symbolic equivalence. OlympiadBench scoring preserves multiple-answer metadata and requires every answer. Numeric answers use the supplied absolute tolerance or $10^{-8}$ by default. MinervaMath numeric answers use relative tolerance $10^{-4}$ in the units requested by the question. These are the implemented grading protocols for the listed dataset distributions. Standalone repeated evaluation uses TP1, whereas periodic evaluation uses the training run's BF16 evaluation layout. Comparisons retain the same evaluation protocol across precision arms.

The final checkpoints are iteration 600 for 4B and iteration 400 for 8B. Both comparisons use evaluation seeds 42/43/44. Each model uses the same evaluation sources across rollout configurations. The Mean column first averages all six tasks within each seed and then computes the mean and sample standard deviation across seeds.

\subsection{Workload Traces}
\label{app:trace-protocol}

Both schedulers use four H200 GPUs with two TP2 replicas. The response quota and sampling settings match Table~\ref{tab:app-training}. The policy remains frozen, including across simulated update boundaries, so collection dynamics can be measured without concurrent changes in the learned policy. We pin the BF16 pool to 83,022 pages per GPU for 4B and 79,639 pages for 8B. Both models have $\kappa=73{,}728$ nominal BF16 KV bytes per context position per GPU under TP2. These pools hold 1,328,352 and 1,274,224 positions respectively. The maximum active batch is 1,024 per replica for either scheduler.

We discard three collection rounds to form the carry pool before measurement. The 4B summaries use the following twelve rounds. The 8B summaries use the following six rounds at every checkpoint. The early 8B run contains twelve measured rounds, of which the first six form the matched comparison. Sampling uses seed 6101. Statistics combine pure-decode calls from both replicas and exclude prefill calls. Active batch and nominal resident volume are averaged over calls. Per-engine decode throughput is the total committed decode tokens divided by summed CUDA-event decode time. It therefore excludes time spent on prefill and collection outside these calls.

The resident-volume calculation is $V_t=\kappa\sum_{i\in\mathcal R_t}\ell_i(t)$ for the active sequences. It counts the nominal KV representation of their contexts and differs from total reserved VRAM or physically measured HBM traffic. The carried contribution is the context belonging to requests resumed from an earlier collection round. Fresh requests supply the remaining contribution. Slots may be reused or compacted without changing this classification.

Fig.~\ref{fig:concentration} displays 17,000 decode calls from one 4B-middle replica. Its interval contains a complete synchronous round and the start of the next one. The first displayed synchronous and asynchronous collection intervals contain 16,379 and 1,759 calls respectively. Display smoothing averages at most 65 calls within a round and samples the resulting curve every eight calls. Batch and context sums are smoothed together so that their product remains the displayed resident volume. Means and reported ratios use unsmoothed calls across the full observation population. The displayed brackets therefore describe particular rounds rather than the average reduction in decode calls.

\subsection{Roofline Comparison and Decode Replay}
\label{app:roofline}
\label{app:decode-replay}

The workload comparison in Fig.~\ref{fig:rollout-quantization}(a) uses nominal BF16 weight bytes and linear-layer FLOPs alongside the traced KV volume. For Qwen3-4B with TP2, the linear projections and output head contribute $W=4.022272$ GB of weights per GPU and $F=4.022272$ GFLOPs per active token per GPU. We compare the demands $W/\mathcal B$, $FB_t/\mathcal F$ and $V_t/\mathcal B$ using peak bandwidth $\mathcal B=4.8$ TB/s and dense BF16 throughput $\mathcal F=989.5$ TFLOPs/s.\footnote{H200 specifications: \url{https://www.nvidia.com/en-us/data-center/h200/}. The dense throughput is half the sparsity-enabled figure.} The aligned axes expose relative demands. They omit kernel inefficiency and communication, which are included in the subsequent execution measurements.

For each checkpoint, we divide its asynchronous trace into five equal-baseline-time phases and select a representative batch and context shape from each phase. Each shape uses uniform context length across its batch. Table~\ref{tab:app-replay-shapes} records the tested shapes. The comparison holds them fixed across precision configurations. It therefore measures execution speedup at fixed concurrency rather than the additional admission enabled by lower cache memory.

\begin{table}[htbp]
  \centering
  \small
  \caption{Representative replay shapes $(B,\ell)$ in generation-phase order. Context lengths are in tokens. The five phases have equal weight in baseline decode time. The first 8B-late context was reduced from the sampled 3,463 to 3,461 tokens to fit the pinned page budget.}
  \label{tab:app-replay-shapes}
  \setlength{\tabcolsep}{3pt}
  \begin{tabular}{@{}lrrrrr@{}}
    \toprule
    Checkpoint & Phase 1 & Phase 2 & Phase 3 & Phase 4 & Phase 5 \\
    \midrule
    4B--100 & 608/621 & 576/798 & 448/1,033 & 304/1,384 & 168/1,963 \\
    4B--300 & 640/874 & 576/1,170 & 464/1,506 & 336/1,968 & 200/2,627 \\
    4B--500 & 640/1,463 & 544/1,895 & 400/2,568 & 272/3,518 & 176/4,767 \\
    8B--100 & 624/547 & 592/676 & 496/879 & 352/1,134 & 192/1,529 \\
    8B--250 & 656/1,863 & 384/3,271 & 320/3,908 & 272/4,435 & 208/5,305 \\
    8B--400 & 368/3,461 & 256/4,712 & 224/5,529 & 200/6,196 & 176/6,972 \\
    \bottomrule
  \end{tabular}
\end{table}

\paragraph{Kernel configurations.}
W4A16 uses vLLM 0.18.0's Machete kernel at the tested batch sizes. FP8 W8A8 uses our schedule-tuned adaptation of vLLM's CUTLASS FP8 GEMM. A separate W8A8 configuration uses rotated INT8 activations. Integer KV8 and KV4 use our CUDA paged-attention kernels with BF16 linear layers. We also measure FlashInfer FA2 attention over an FP8 KV cache within the same engine. This last configuration tests a different cache format and attention path, rather than a separate vLLM end-to-end system. The integer KV4 path enables the 8-bit PV operand described in Appendix~\ref{app:kv4-implementation}.

\paragraph{Timing and aggregation.}
After warmup, we collect 90 CUDA-event samples per state in three within-process blocks. Timing covers the captured model forward with TP communication and the captured argmax operation. Conversion work required by the selected numerical path is included. Scheduler work and stochastic sampling outside the graph are excluded together with prefill and learning. If $\tau_{q,j}$ denotes mean latency for configuration $q$ at state $j$, the aggregate is
\begin{equation}
  S_q=\left(\sum_{j=1}^{5}\omega_j
            \frac{\tau_{q,j}}{\tau_{\mathrm{BF16},j}}\right)^{-1},
  \qquad \omega_j=\frac15.
  \label{eq:app-speedup}
\end{equation}
This is a baseline-time-weighted harmonic mean of state speedups. Replacing state means with medians changes every 4B aggregate by less than $0.006\times$ and every 8B aggregate by less than $0.002\times$. Repeated BF16 middle-checkpoint runs agree within $0.3\%$ after aggregation. These repeats assess timing consistency within the replay protocol.

\subsection{Shared-Prefix Fidelity and Protection Timing}
\label{app:fidelity-protocol}

For each fidelity checkpoint, we generate one BF16 response to each of 32 fixed prompts using temperature one. Eight positions span each response at generated-prefix indices $\operatorname{round}(j(L-1)/7)$ for $j=0,\ldots,7$. The resulting corpus has 256 states per checkpoint. We force every sampler along these same token histories and capture full-vocabulary logits at the selected positions. Both the BF16 reference and the quantized configurations run in the same TP2 engine. CPU analysis applies float64 log-softmax to the recorded logits before temperature scaling or sampling truncation and computes $D_{\mathrm{KL}}(p_{\mathrm{BF16}}\|q)$.

The reported mean gives each state equal weight. Confidence intervals use 10,000 bootstrap samples of the 32 response paths, retaining all eight positions of a selected path together. An independent BF16 process repeat at 4B--300 produces bit-identical reference logits and zero KL. This control measures the repeatability of the comparison within the same engine. The conditional KL remains a sampler diagnostic, while training measures the effect on learned performance.

Protection timing uses heterogeneous contexts to retain the variation encountered by the attention kernel. At 4B--300, we reconstruct 24 batch-context vectors from the midpoint of each measured round across both replicas. They contain 407--475 active responses with mean contexts of approximately 1.4k--1.7k tokens. At 8B, we construct 11 heterogeneous vectors from the mean and quantiles recorded in the iteration-250 traces. These are synthetic vectors because preemption prevents exact reconstruction from contiguous residency records. They contain 260--374 responses with mean contexts of approximately 3.3k--4.7k tokens. The 8B protection kernels and fidelity probe use the iteration-200 checkpoint. Thus its workload shapes and policy parameters come from different recorded iterations.

Each protection arm replays the same vectors as its BF16 reference with ten warmup calls and 90 timed calls per state. We use the median latency for each state and aggregate as $S_q^{\mathrm{ragged}}=\sum_j\tau_{\mathrm{BF16},j}/\sum_j\tau_{q,j}$. This weights states by their BF16 replay duration. These heterogeneous-state measurements have a different state population and aggregation from Eq.~\ref{eq:app-speedup}. Comparisons of protection overhead use the KV4 reference measured in this same heterogeneous-state suite.

\subsection{Rollout and Training Time}
\label{app:timing-boundaries}

GPU decode timing isolates execution once a batch is ready. Rollout collection time additionally includes fresh and carried-prefix prefill along with scheduling and output collection. In the training logs, the rollout timer also encloses reward evaluation performed by the collector. The iteration timer starts before switching to the rollout engine and ends after optimization. It includes parameter synchronization and both phase transitions. Logged training timers take the maximum across ranks. Evaluation and checkpoint saving occur outside that timer. Table~\ref{tab:app-timers} summarizes these boundaries.

Decode replay fixes the active batch and context shapes across precision configurations, so it measures only the execution benefit at a matched resident workload. End-to-end collection instead retains each configuration's precision-specific KV pool and the resulting scheduler decisions. Its timing therefore includes both the reduction in per-forward execution cost and any change in admitted concurrency caused by the cache-capacity constraint. Because each learned policy also produces its own responses, response-length trajectories are reported alongside these timings to distinguish capacity and execution gains from changes in generated work.

\begin{table}[htbp]
  \centering
  \small
  \caption{Timing scopes. Quantization can change forward cost, feasible resident concurrency, and the amount of work completed in a collection round. These measurements therefore answer different performance questions.}
  \label{tab:app-timers}
  \begin{tabular}{@{}p{0.25\linewidth}p{0.69\linewidth}@{}}
    \toprule
    Measurement & Included work \\
    \midrule
    Decode replay & Captured forward, TP communication, and argmax at fixed states \\
    Decode throughput & Committed tokens per summed GPU decode-service time \\
    Rollout collection & Collection including prefill/re-prefill and reward scoring \\
    Training iteration & Collection and learning with weight synchronization and phase switches \\
    \bottomrule
  \end{tabular}
\end{table}

Section~\ref{sec:5} sums the iteration and rollout timers along the same checkpoint lineage used for policy evaluation. The 8B comparison contains exactly 400 updates per configuration. If a run restarts from a saved checkpoint, only the branch leading to the evaluated policy contributes to the total. Initialization and restart downtime are excluded together with evaluation and saving. These totals measure cumulative execution of the retained training iterations.

The reported speedup divides the sum of BF16 times by the corresponding sum for each sampler. It is not the mean of per-iteration ratios. The 8B full-run comparisons use iterations 1--400. Its early interval contains iterations 1--100 and the final evaluation interval contains 351--400. Response-length statistics use the same populations. Actual policies generate their own responses, so these end-to-end comparisons complement the fixed-workload replay rather than replacing it.

The iteration breakdown groups the logged rollout and policy-update timers with the sum of policy and reference log-probability timers. The remaining difference from the total iteration timer is assigned to other work. This residual includes phase transitions and work outside the three named groups. Component smoothing uses a trailing 20-iteration arithmetic mean and preserves their sum. Reported time shares divide unsmoothed component sums by the unsmoothed total. They describe the change in time allocation after quantization rather than a monotonic decrease in rollout share over training.

\section{Analysis of KV Concentration}
\label{app:kv-concentration}

This appendix proves Proposition~\ref{prop:kv-concentration} from a renewal-reward argument and extends the analysis to variable prompt lengths and finite collection. We first state the required renewal-theoretic result and include its proof for completeness. We then instantiate its cycle length and reward for continuous-refill and synchronous rollout. Throughout the appendix, time is measured in decode steps. Prefill, carried-prefix reconstruction, and policy updates lie outside this clock.

\subsection{Renewal-Reward Preliminaries}
\label{app:identity}

\begin{definition}[Renewal-reward process]
Let $\{(X_n,R_n)\}_{n\geq1}$ be an i.i.d. sequence of random pairs, where $X_n\in\mathbb N_{>0}$ is the duration of cycle $n$ and $R_n\in\mathbb R$ is its reward. The reward $R_n$ may depend on the duration $X_n$. Define the renewal epochs $S_0=0$ and $S_n=\sum_{j=1}^n X_j$, and let
$N(t)=\max\{n\geq0:S_n\leq t\}$
be the number of cycles completed by time $t$. The cumulative reward from completed cycles is
$Y(t)=\sum_{n=1}^{N(t)}R_n$.
\end{definition}

The following is the almost-sure part of the renewal-reward theorem stated by \citet[Theorem~3.6.1]{ross1996stochastic}. We include the argument to make its use in our setting explicit.

\begin{theorem}[Renewal-reward theorem]
\label{thm:renewal-reward}
Suppose that $0<\mathbb E[X_1]<\infty$ and $\mathbb E[|R_1|]<\infty$. Then
\begin{equation}
    \lim_{t\to\infty}\frac{Y(t)}{t}
    =
    \frac{\mathbb E[R_1]}{\mathbb E[X_1]}
    \qquad\text{almost surely}.
    \label{eq:app-renewal-reward}
\end{equation}
No independence between $X_n$ and $R_n$ within a cycle is required.
\end{theorem}

\begin{proof}
By the strong law of large numbers, $S_n/n\to\mathbb E[X_1]$ almost surely. Since $S_{N(t)}\leq t<S_{N(t)+1}$ and $N(t)\to\infty$, it follows that $N(t)/t\to1/\mathbb E[X_1]$ almost surely. Applying the strong law again to the rewards gives
\[
    \frac{Y(t)}{t}
    =
    \frac{\sum_{n=1}^{N(t)}R_n}{N(t)}
    \frac{N(t)}{t}
    \longrightarrow
    \mathbb E[R_1]\,
    \frac{1}{\mathbb E[X_1]},
\]
which proves Eq.~\ref{eq:app-renewal-reward}.
\end{proof}

The theorem counts rewards from completed cycles. In the rollout application, an observation may end while a response or synchronous wave is still active. The following lemma shows that this terminal contribution does not affect the long-run average.

\begin{lemma}[Negligible terminal cycle]
\label{lem:terminal-cycle}
Under the assumptions of Theorem~\ref{thm:renewal-reward}, suppose that the absolute reward accumulated during cycle $n$ is bounded by a nonnegative random variable $Z_n$, where $\{(X_n,Z_n)\}_{n\geq1}$ is i.i.d. and $\mathbb E[Z_1]<\infty$. Let $\widetilde Y(t)$ include the reward accrued in the possibly incomplete cycle at time $t$. Then
\[
    \frac{\widetilde Y(t)-Y(t)}{t}
    \longrightarrow 0
    \qquad\text{almost surely}.
\]
Consequently, $\widetilde Y(t)/t$ has the same limit as Eq.~\ref{eq:app-renewal-reward}.
\end{lemma}

\begin{proof}
For every $\epsilon>0$, integrability of $Z_1$ implies
$\sum_{n\geq1}\Pr(Z_n>\epsilon n)<\infty$.
The Borel--Cantelli lemma therefore gives $Z_n/n\to0$ almost surely. Since the incomplete-cycle reward is bounded by $Z_{N(t)+1}$,
\[
    \frac{|\widetilde Y(t)-Y(t)|}{t}
    \leq
    \frac{Z_{N(t)+1}}{N(t)+1}
    \frac{N(t)+1}{t}
    \longrightarrow 0,
\]
where the first factor converges to zero and the second converges to $1/\mathbb E[X_1]$.
\end{proof}

\subsection{Continuous Refill and Length-Biased Occupancy}
\label{app:continuous-refill}

A response with prompt length $P$ and generation length $L$ visits cached contexts of lengths $P,P+1,\ldots,P+L-1$. Its total logical KV work is
$Q(P,L)=PL+L(L-1)/2$.
For fixed $P$, let $\mu=\mathbb E[L]$ and
$q=\mathbb E[Q(P,L)]
=P\mu+\mathbb E[L(L-1)]/2$.
The assumption $\mathbb E[L^2]<\infty$ ensures $q<\infty$. This reward counts the context positions logically referenced during decoding. Physical memory traffic additionally depends on the cache representation and attention implementation.

Consider one continuously refilled slot. Each completed response forms a renewal cycle with duration $X_n=L_n$ and reward $R_n=Q(P,L_n)$. Theorem~\ref{thm:renewal-reward} therefore gives a completed-cycle reward rate of $q/\mu$. The context processed by an incomplete response is bounded by its total logical work $Q(P,L)$, which has finite expectation, so Lemma~\ref{lem:terminal-cycle} shows that including the active response does not change the limit. Hence the long-run context processed per decode step in one slot is $q/\mu$.

Because continuous refill maintains $B$ occupied slots, the same argument applies to each slot, and the finite sum of their almost-sure limits gives
\begin{equation}
    \overline\ell_{\mathrm{async}}
    =
    \frac{q}{\mu}
    =
    P+\frac{\mathbb E[L(L-1)]}{2\mathbb E[L]},
    \qquad
    \overline V_{\mathrm{async}}
    =
    \frac{\kappa Bq}{\mathbb E[L]}.
    \label{eq:app-refill}
\end{equation}

The same renewal-reward construction also derives the length-biased distribution observed at decode steps. Fix a response length $l$ and an age $a\in\{0,\ldots,l-1\}$. Assign one unit of reward to a cycle if it has length $l$ and visits age $a$, and zero reward otherwise. Each response of length $l$ visits age $a$ exactly once, so the expected reward per cycle is $\Pr(L=l)$. Theorem~\ref{thm:renewal-reward} gives the long-run fraction of slot-steps with response length $l$ and age $a$ as $\Pr(L=l)/\mu$. Summing over the $l$ possible ages and then conditioning on $L^\ast=l$ yields
\begin{equation}
    \Pr(L^\ast=l)
    =
    \frac{l\Pr(L=l)}{\mu},
    \qquad
    \Pr(g=a\mid L^\ast=l)
    =
    \frac{1}{l},
    \quad 0\leq a<l.
    \label{eq:app-inspection}
\end{equation}
Thus, longer responses are observed more frequently because they occupy a slot for more decode steps, an instance of inspection-paradox length bias~\citep{ross1996stochastic}. Averaging the age under Eq.~\ref{eq:app-inspection} gives
$\mathbb E[g]=\mathbb E[L(L-1)]/(2\mu)$,
which recovers the first identity in Eq.~\ref{eq:app-refill}. Equivalently, if $\operatorname{Var}(L)=\sigma^2$, then
$\overline\ell_{\mathrm{async}}
=P+(\mu-1)/2+\sigma^2/(2\mu)$.

\subsection{Synchronous Waves and Proof of Proposition~\ref{prop:kv-concentration}}
\label{app:proof}

A synchronous wave contains $B$ independently sampled responses and lasts
$M_B=\max_{1\leq i\leq B}L_i$
decode steps. Treat each wave as a renewal cycle with duration $X_n=M_{B,n}$ and resident-volume reward
$R_n=\kappa\sum_{i=1}^{B}Q(P,L_{n,i})$.
The duration and reward of a wave are generally dependent, but Theorem~\ref{thm:renewal-reward} requires only that the pairs $(X_n,R_n)$ be i.i.d. across waves.

The expected reward of one wave is $\kappa Bq$, while its expected duration is $\mathbb E[M_B]$. Both are finite because $M_B\leq\sum_{i=1}^{B}L_i$ and $q<\infty$. The reward accumulated in a partial terminal wave is bounded by its total wave reward, so Lemma~\ref{lem:terminal-cycle} again applies. Theorem~\ref{thm:renewal-reward} therefore yields
\begin{equation}
    \overline V_{\mathrm{sync}}
    =
    \frac{\kappa Bq}{\mathbb E[M_B]}.
    \label{eq:app-sync-volume}
\end{equation}

The step-averaged synchronous occupancy follows from the same construction. A response of length $L_i$ contributes one occupied slot for each of its $L_i$ decode steps, so the total occupancy reward of a wave is $\sum_{i=1}^{B}L_i$. Its expectation is $B\mathbb E[L]$, giving
\begin{equation}
    \overline B_{\mathrm{sync}}
    =
    \frac{B\mathbb E[L]}{\mathbb E[M_B]}.
    \label{eq:app-sync-occupancy}
\end{equation}
Combining Eq.~\ref{eq:app-refill} with Eq.~\ref{eq:app-sync-volume} establishes the KV concentration factor
\begin{equation}
    \mathcal C_{\mathrm{KV}}
    =
    \frac{\overline V_{\mathrm{async}}}
         {\overline V_{\mathrm{sync}}}
    =
    \frac{\mathbb E[M_B]}{\mathbb E[L]}
    \geq 1.
    \label{eq:app-sync}
\end{equation}

The inequality follows from $M_B\geq L_1$ almost surely. If $B\geq2$ and the response-length distribution is non-degenerate, there exist $a<b$ such that $\Pr(L=a)>0$ and $\Pr(L=b)>0$. The event $\{L_1=a,L_2=b\}$ then has positive probability and implies $M_B>L_1$. Consequently, $\mathbb E[M_B]>\mathbb E[L]$ and $\mathcal C_{\mathrm{KV}}>1$.

For completeness, the step-averaged mean context of a synchronous wave can be obtained by assigning the wave the reward
$\sum_{t=0}^{M_B-1}(P+t)=PM_B+M_B(M_B-1)/2$.
Since $M_B^2\leq\sum_{i=1}^{B}L_i^2$, this reward has finite expectation. A final application of Theorem~\ref{thm:renewal-reward} gives
\[
    \overline\ell_{\mathrm{sync}}
    =
    P+
    \frac{\mathbb E[M_B^2]-\mathbb E[M_B]}
         {2\mathbb E[M_B]}.
\]
Together, these identities prove Proposition~\ref{prop:kv-concentration}.

\subsection{Extensions Beyond Continuous Refill}
\label{app:finite-collection}

When batch occupancy varies, the mean context observed in an occupied slot differs from the unweighted mean of the per-step context curve. For a schedule $s$, define the slot-weighted context and step-averaged occupancy as
\begin{equation}
    \ell_s^{\mathrm{slot}}
    =
    \lim_{D\to\infty}
    \frac{
        \sum_{t<D}\sum_{i\in\mathcal R_t^s}\ell_i(t)
    }{
        \sum_{t<D}B_t^s
    },
    \qquad
    \overline B_s
    =
    \lim_{D\to\infty}
    \frac{1}{D}\sum_{t<D}B_t^s.
    \label{eq:app-slot-context}
\end{equation}
The first quantity samples occupied slot-steps, whereas the main-text quantity $\overline\ell_s$ assigns equal weight to every decode step. They coincide under continuous refill because $B_t=B$, but differ when occupancy varies. In particular, the low-occupancy steps in the tail of a synchronous wave receive greater relative weight in $\overline\ell_s$ than in $\ell_s^{\mathrm{slot}}$.

For an i.i.d. stream of prompt and response-length pairs $(P,L)$ satisfying $\mathbb E[PL]<\infty$ and $\mathbb E[L^2]<\infty$, take the renewal-cycle duration to be $L$ and its logical-work reward to be $Q(P,L)$. Applying Theorem~\ref{thm:renewal-reward} gives
\begin{equation}
    \ell^{\mathrm{slot}}
    =
    \frac{\mathbb E[Q(P,L)]}{\mathbb E[L]}
    =
    \frac{\mathbb E[PL]}{\mathbb E[L]}
    +
    \frac{\mathbb E[L(L-1)]}{2\mathbb E[L]}.
    \label{eq:app-variable-prompt}
\end{equation}
The prompt contribution is weighted by response length and reduces to $\mathbb E[P]$ when $P$ and $L$ have zero covariance. Under continuous refill, Eq.~\ref{eq:app-variable-prompt} is also the mean resident context per active slot.

The same workload identity extends to prefix-preserving schedules that pause and resume unfinished responses. Suppose that at most $H<\infty$ admitted responses remain unfinished at any time and that every admitted response eventually completes. After $n$ admissions, the number of unexecuted tokens is bounded by $H\max_{j\leq n}L_j$, while the remaining logical KV work is bounded by $H\max_{j\leq n}Q(P_j,L_j)$. The argument used in Lemma~\ref{lem:terminal-cycle} implies that both residuals are $o(n)$ almost surely under the stated moment assumptions. The strong law then gives Eq.~\ref{eq:app-variable-prompt} for the executed slot-steps. Whenever $\overline B_s$ exists, the resident-volume average follows from $\overline V_s=\kappa\overline B_s\ell_s^{\mathrm{slot}}$. In our rollout setting, $L$ is bounded by the maximum generation length, so the finite-second-moment assumption in Proposition~\ref{prop:kv-concentration} holds automatically.

Actual quota-based collection differs from the idealized continuous-refill model because it operates with a finite request pool and admits requests subject to KV page availability. Its occupancy can therefore decline within a collection interval rather than remaining fixed at $B$. In our measurements, the synchronous replicas collect 512 responses per round, while both schedulers permit up to 1,024 active requests per replica, so the practical comparison does not exactly instantiate the common-$B$ construction in Proposition~\ref{prop:kv-concentration}. Policy updates can also change the response-length distribution across iterations. The proposition isolates the scheduling mechanism for a fixed workload, while the frozen-checkpoint traces in Fig.~\ref{fig:concentration} measure the resulting KV concentration under finite request pools, page constraints, and the actual quota-based scheduler.
\section{Supplementary Results}
\label{app:supplementary-results}

This appendix gives additional workload traces, 8B quantization replay, protection and storage-format ablations, multi-task learning curves, endpoint pass@$k$ results, 4B training curves, and an iteration-time breakdown. All measurements follow Appendix~\ref{app:experimental-setup}.

\subsection{Workload and Acceleration across Training}
\label{app:phase-results}

Fig.~\ref{fig:app-workload-phases} extends the 4B-middle trace in Fig.~\ref{fig:concentration} to earlier and later checkpoints of both models. The resident-KV ratio closely follows the active-batch ratio at each checkpoint. Their magnitudes vary substantially across training, even though the scheduler configuration remains fixed. Larger resident contexts also increase per-step cost, so decode-throughput gains are smaller than the gains in concurrency.

\begin{figure}[htbp]
  \centering
  \includegraphics[width=0.85\linewidth]{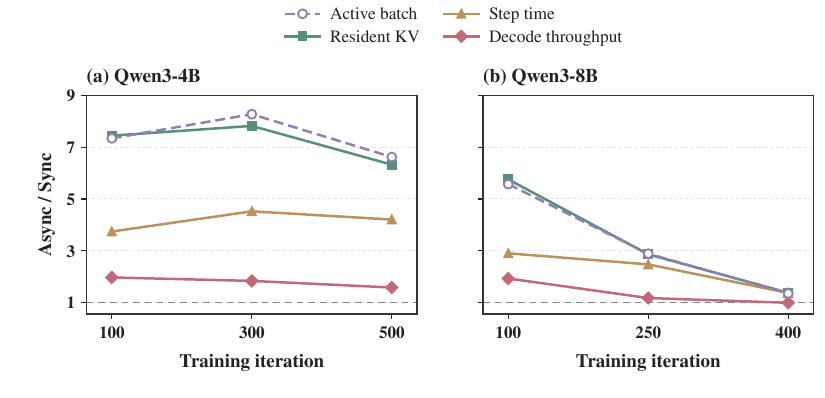}
  \vspace{-5pt}
  \caption{Workload concentration across training for Qwen3-4B-Base and Qwen3-8B-Base. Each point reports the ratio of asynchronous to synchronous collection at a frozen BF16 checkpoint. Active batch size and nominal resident KV volume are averaged over pure-decode calls. Decode-step latency is measured with CUDA events, and decode-service throughput is the number of committed decode tokens divided by summed decode time. Both TP2 replicas are pooled after three warmup rounds, with twelve measured rounds at 4B and six at 8B. The dashed line denotes parity between the two schedulers.}
  \label{fig:app-workload-phases}
  \vspace{-10pt}
\end{figure}

The late 8B checkpoint illustrates the capacity constraint. Its asynchronous engine holds an average of 91.7 GB of active KV per GPU against a nominal pool capacity of 93.9 GB. The synchronous workload is also substantial at 66.8 GB. Asynchronous scheduling therefore provides a smaller concurrency increase than at the early checkpoint, and its decode throughput is $0.98\times$ the synchronous value. The fixed page budget thus limits the benefit of scheduling at this checkpoint. This near-capacity regime also creates a second opportunity for KV compression: a smaller cache can admit more resident context before page exhaustion restricts concurrency. The fixed-workload replay below isolates execution speedup and does not include this admission benefit, whereas the end-to-end training measurements retain the precision-specific pool capacities.

\subsection{Quantization Replay on Qwen3-8B-Base}
\label{app:8b-quantization}

The fixed-workload comparison below extends the 4B results in Section~\ref{sec:3.2} to 8B. KV quantization remains effective even where the additional decode throughput from asynchronous scheduling has become small.

\begin{figure}[htbp]
  \centering
  \includegraphics[width=0.85\linewidth]{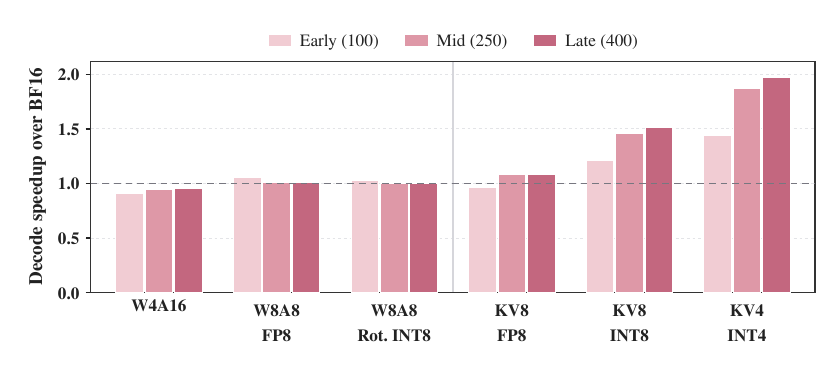}
  \vspace{-5pt}
  \caption{Fixed-workload decode-forward speedup for Qwen3-8B-Base on H200 with TP2. Early, Mid, and Late correspond to iterations 100, 250, and 400. Each result aggregates five matched batch--context shapes using equal BF16-time weights, and the dashed line denotes BF16 parity. W4A16 uses vLLM Machete, FP8 W8A8 uses our adaptation of the vLLM CUTLASS GEMM, FP8 KV8 uses FlashInfer FA2, and integer KV8 and KV4 use our paged-attention kernels. KV4 denotes the unprotected INT4 cache. Appendix~\ref{app:decode-replay} specifies the execution paths and timing boundaries.}
  \label{fig:app-8b-quantization}
  \vspace{-10pt}
\end{figure}

The same pattern holds at 8B (Fig.~\ref{fig:app-8b-quantization}). KV4 provides $1.44\times$ decode speedup at the early checkpoint and $1.97\times$ at the late checkpoint, while FP8 W8A8 provides at most $1.06\times$. KV8 also illustrates why storage precision alone is insufficient to predict performance: the FlashInfer FP8 path gives $0.97$--$1.09\times$ speedup, compared with $1.22$--$1.52\times$ for our integer path. These configurations change both representation and attention arithmetic.

\FloatBarrier
\subsection{Location and Precision of Protection}
\label{app:band}

Fig.~\ref{fig:app-protection} separates the effect of protection location from its execution cost. Panels~(a,c) compare allocations of the same 144 protected positions against KV4. At 4B--300, protecting sink 16 and recent 128 reduces mean KL by $7.0\times$ relative to KV4. It also produces lower mismatch than assigning all 144 positions to either region. The same allocation reduces mean KL by $5.7\times$ at the 8B fidelity checkpoint. Both regions therefore contribute to the default protection in these comparisons.

\begin{figure}[htbp]
  \centering
  \includegraphics[width=0.85\linewidth]{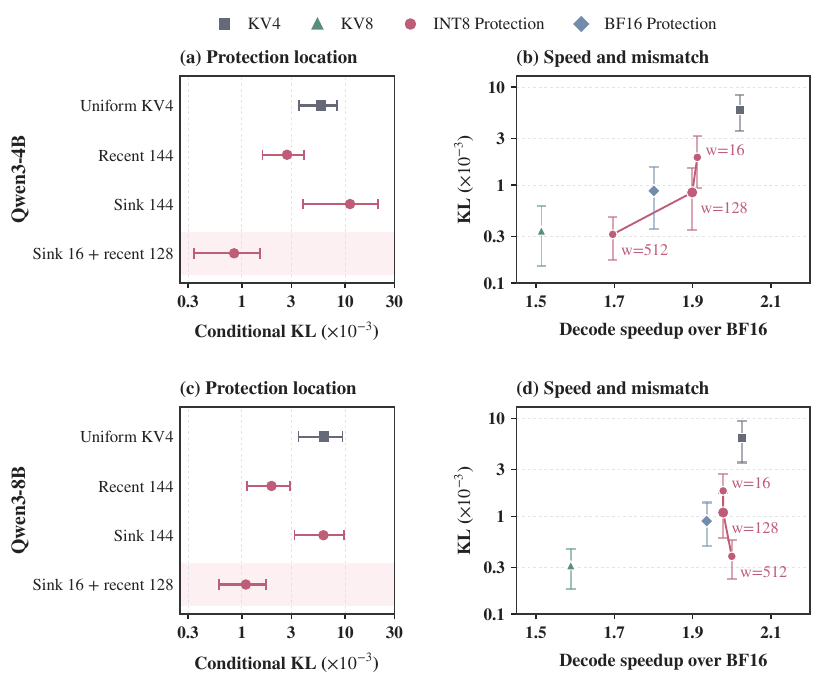}
  \vspace{-5pt}
  \caption{Fidelity and decode cost of selective protection. Panels~(a,c) compare allocations of the same 144 INT8-protected positions at 4B and 8B. Panels~(b,d) report BF16-relative decode speedup against shared-prefix $D_{\mathrm{KL}}(p_{\mathrm{BF16}}|q)$. The connected INT8 series fixes the sink width at 16 and varies the recent width $w$, with the enlarged marker denoting the default $w=128$. BF16 protection uses the same 16/128 positions, while KV8 stores the full cache in INT8. KL axes are logarithmic, and KL error bars are 95\% response-cluster bootstrap intervals over 32 responses. Speedups use matched heterogeneous decode states. The 4B measurements use the iteration-300 policy and traces. The 8B fidelity measurements use the iteration-200 policy, with timing states matched to iteration-250 trace statistics.}
  \label{fig:app-protection}
  \vspace{-10pt}
\end{figure}

The speed--mismatch plots in Fig.~\ref{fig:app-protection}(b,d) show how this reduction relates to execution cost. The default retains approximately $94\%$ of KV4 decode throughput at 4B and $98\%$ at 8B in this suite. BF16 protection is slower than INT8 protection at the same positions. Their KL intervals overlap, with a slightly lower BF16-protection mean at 8B and a slightly higher mean at 4B. This comparison changes protected QK arithmetic along with storage precision. The observed trade-off supports INT8 protection as an efficient choice without requiring exact equivalence to BF16 protection.

The recent width controls how much context receives additional protection. Increasing it to 512 further reduces mismatch in both probes. Its execution cost is larger at 4B, while the 8B timings remain close to those of the default. Conversely, a 16-token recent window retains speed but leaves more mismatch. These results establish a useful protection range and explain the choice of a compact default. The training comparisons in Section~\ref{sec:4} and Section~\ref{sec:5} assess its effect on learning.

\FloatBarrier
\subsection{Quantization Format and Rotation}
\label{app:quantization-design}

We examine the numerical choices underlying the INT4 representation with a controlled storage-quantization probe. It follows the same BF16-generated histories as the deployed-engine fidelity comparison and measures the conditional KL in Eq.~\ref{eq:sampler-kl}. Prompt prefill remains BF16 and the forced continuation uses the quantized cache. The probe reconstructs the stored values and performs attention with BF16 operands. It therefore compares storage formats separately from the INT8 query and probability operands used by the deployed kernel. Its absolute KL values need not equal those in Fig.~\ref{fig:protection}(b,c).

\begin{figure}[htbp]
  \centering
  \includegraphics[width=0.95\linewidth]{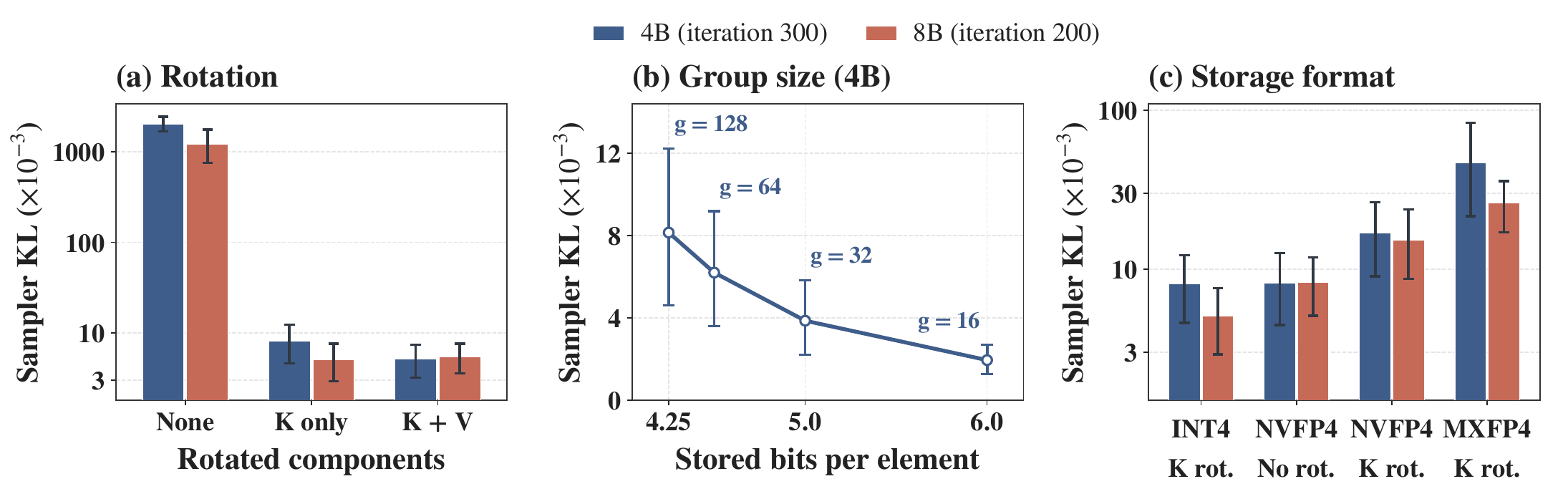}
  \vspace{-5pt}
  \caption{Cache-representation fidelity ablations for the quantized KV representation. (a) Rotation choices for affine INT4 with group size 128. (b) Quantization group size versus storage cost at 4B under K-only rotation. (c) Affine INT4, NVFP4, and MXFP4 representations with the rotation applied to each format indicated below the corresponding result. All probes reconstruct the stored cache and use BF16 attention operands, thereby excluding the integer query and probability arithmetic of the deployed kernels. Storage cost includes codes and metadata. Panels~(a,c) use logarithmic KL axes, and error bars are 95\% response-cluster bootstrap intervals over 32 histories.}
  \label{fig:app-quantization-design}
  \vspace{-10pt}
\end{figure}

Fig.~\ref{fig:app-quantization-design}(a) shows that K rotation is important when a full head shares one affine scale. It sharply reduces mismatch at both checkpoints. Adding V rotation produces a smaller mean at 4B but does not give a consistent improvement across the probed histories and checkpoints. The response-paired geometric KL ratio for K-and-V rotation relative to K-only rotation is $0.89$ with a 95\% interval of $[0.55,1.53]$ at 4B. At 8B--200, the corresponding ratio is $1.29$ with interval $[0.63,3.04]$. These comparisons support retaining K-only rotation in the deployed design.

As shown in Fig.~\ref{fig:app-quantization-design}(b), finer channel groups improve fidelity at a higher storage cost. Reducing the affine group size from 128 to 32 lowers the 4B mean KL from $8.14\times10^{-3}$ to $3.86\times10^{-3}$ while increasing storage from 4.25 to 5.00 bits per element. A further reduction to group size 16 lowers mean KL to $1.94\times10^{-3}$ at 6.00 bits per element. Each group also requires its own scale and offset in the attention reduction. This comparison motivates examining selective protection as another use of additional precision. It does not establish a latency ordering between the deployed protection kernel and a separately optimized small-group kernel.

Fig.~\ref{fig:app-quantization-design}(c) compares INT4 with NVFP4 and MXFP4 representations. The NVFP4 probe uses E2M1 values with 16-element E4M3 scales and an FP32 global scale per token--head vector. This implementation uses 4.75 bits per element, with the global scale accounting for 0.25 bits per element. The MXFP4 probe uses 32-element E8M0 scales and occupies 4.25 bits per element. Without rotation, NVFP4 has a 4B mean KL close to head-grouped INT4 with K rotation. Adding K rotation raises NVFP4 mismatch at both checkpoints, showing that its effect depends on the representation. Thus the choice of INT4 rests on its head-level affine structure and validated integer execution rather than a universal fidelity advantage over FP4. No FP4 attention-kernel latency is inferred from this numerical probe.

\FloatBarrier
\subsection{Learning Trajectories and Endpoint Evaluation}
\label{app:protection-learning}

Fig.~\ref{fig:app-protection-learning} extends the MATH-500 comparison to both AIME benchmarks. Each precision configuration uses one Qwen3-8B-Base training run with seed 43 and ten warmup iterations. Training continues to iteration 400 under the same GRPO and TIS settings. Every checkpoint is evaluated in BF16. Each reported mean combines the evaluation recorded by the training run with standalone evaluations using seeds 42 and 44. The standard deviation summarizes these three evaluations at fixed policy parameters.

\begin{figure}[htbp]
  \centering
  \includegraphics[width=0.95\linewidth]{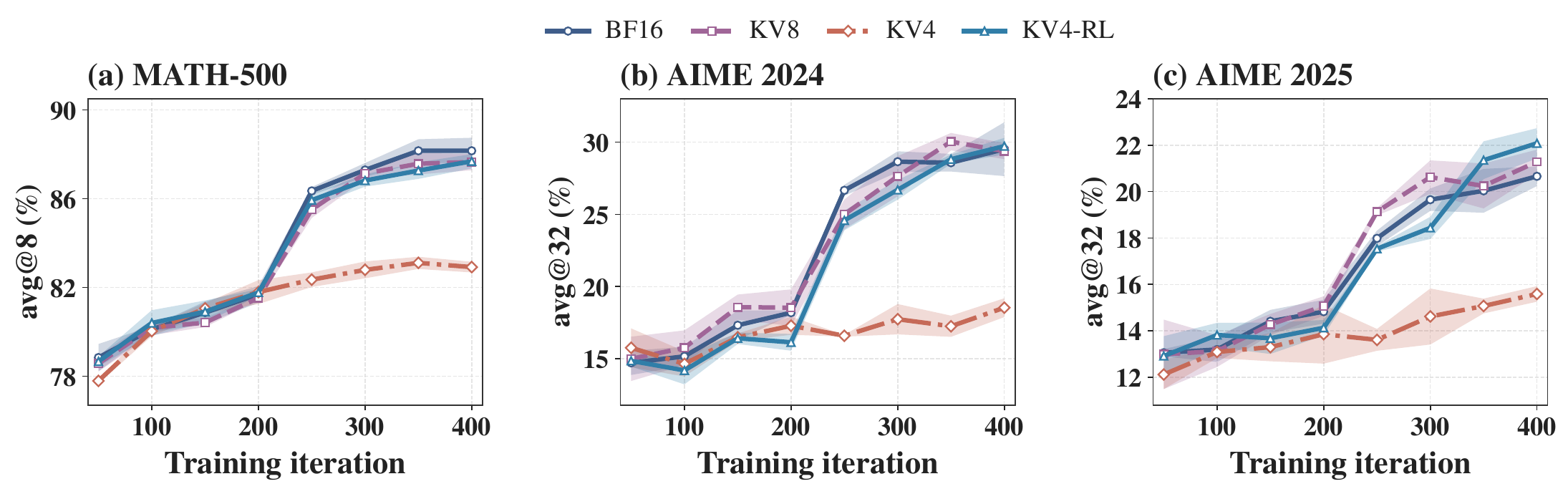}
  \vspace{-5pt}
  \caption{Periodic BF16 evaluation of Qwen3-8B-Base policies trained with four rollout precisions. (a) MATH-500 avg@$8$. (b) AIME 2024 avg@$32$. (c) AIME 2025 avg@$32$. Each rollout precision uses one training run, and every checkpoint is evaluated with seeds 42, 43, and 44. Curves report the mean across these evaluations, with shading denoting one sample standard deviation.}
  \label{fig:app-protection-learning}
  \vspace{-10pt}
\end{figure}

KV4 improves early but falls behind after iteration 200 as BF16 and KV4-RL continue to improve. KV4-RL remains close to BF16 on MATH-500 and AIME 2024 and finishes higher on AIME 2025 in this run. The common BF16 evaluator attributes these differences to the learned policies rather than their test-time cache precision.

Table~\ref{tab:app-protection-endpoints} complements the avg@$k$ results in Table~\ref{tab:final-performance} with pass@$k$ on the same six benchmarks for both model sizes. All scores use the BF16 evaluation protocol in Appendix~\ref{app:evaluation}. Task means and sample standard deviations are computed over three evaluation seeds. For the Mean column, we first average the six tasks within each seed and then compute the mean and standard deviation across seeds. The 8B comparison combines in-training and standalone evaluation, while all three 4B repeats are standalone evaluations. Each model retains the same collection of evaluation sources across rollout precisions.

\begin{table}[htbp]
  \centering
  \caption{Final Qwen3-Base pass@$k$ accuracy (\%) under BF16 evaluation. The 4B and 8B checkpoints are evaluated after 600 and 400 training iterations, respectively. MATH-500, OlympiadBench, and MinervaMath use pass@$8$, while AIME 2024/2025 and AMC 2023 use pass@$32$. Subscripts denote standard deviation over three evaluation seeds. Mean assigns equal weight to the six tasks within each seed before computing the reported mean and standard deviation. All values are rounded to one decimal place.}
  \label{tab:app-protection-endpoints}
  \vspace{3pt}
\begingroup
\small
\setlength{\tabcolsep}{3pt}
\renewcommand{\arraystretch}{1.0}
\resizebox{0.95\linewidth}{!}{%
\begin{tabular}{@{}lcccccccc@{}}
\toprule
& & \multicolumn{3}{c}{pass@$8$} & \multicolumn{3}{c}{pass@$32$} & \\
\cmidrule(lr){3-5}\cmidrule(lr){6-8}
\textbf{Model} & \textbf{Rollout} & \textbf{MATH-500} & \textbf{OlympiadBench} & \textbf{MinervaMath} & \textbf{AIME24} & \textbf{AIME25} & \textbf{AMC23} & \textbf{Mean} \\
\midrule
\multirow{4}{*}{Qwen3-8B-Base} & BF16 & $96.5_{\pm 0.3}$ & $73.8_{\pm 0.5}$ & $70.7_{\pm 2.4}$ & $66.7_{\pm 6.7}$ & $52.2_{\pm 5.1}$ & $97.5_{\pm 0.0}$ & $76.2_{\pm 1.9}$ \\
 & KV8 & $96.5_{\pm 0.1}$ & $73.6_{\pm 0.3}$ & $69.6_{\pm 1.5}$ & $64.4_{\pm 3.8}$ & $55.6_{\pm 1.9}$ & $96.7_{\pm 1.4}$ & $76.1_{\pm 0.6}$ \\
 & KV4 & $92.3_{\pm 0.6}$ & $66.3_{\pm 0.7}$ & $69.0_{\pm 0.8}$ & $56.7_{\pm 3.3}$ & $37.8_{\pm 1.9}$ & $93.3_{\pm 3.8}$ & $69.2_{\pm 0.4}$ \\
\noalign{\vskip 1.5pt}
\cdashline{2-9}[4pt/2pt]
\noalign{\vskip 1.5pt}
 & KV4-RL & $96.2_{\pm 0.2}$ & $74.2_{\pm 0.7}$ & $71.3_{\pm 2.9}$ & $64.4_{\pm 1.9}$ & $57.8_{\pm 1.9}$ & $96.7_{\pm 1.4}$ & $76.8_{\pm 1.1}$ \\
\midrule
\multirow{4}{*}{Qwen3-4B-Base} & BF16 & $92.5_{\pm 0.6}$ & $66.8_{\pm 1.1}$ & $63.8_{\pm 0.2}$ & $45.6_{\pm 3.8}$ & $42.2_{\pm 3.8}$ & $95.0_{\pm 2.5}$ & $67.7_{\pm 1.2}$ \\
 & KV8 & $92.3_{\pm 0.6}$ & $66.9_{\pm 0.6}$ & $65.4_{\pm 1.0}$ & $56.7_{\pm 8.8}$ & $37.8_{\pm 5.1}$ & $95.8_{\pm 3.8}$ & $69.1_{\pm 1.6}$ \\
 & KV4 & $92.0_{\pm 0.4}$ & $65.8_{\pm 1.5}$ & $64.8_{\pm 0.2}$ & $44.4_{\pm 3.8}$ & $41.1_{\pm 7.7}$ & $93.3_{\pm 1.4}$ & $66.9_{\pm 1.8}$ \\
\noalign{\vskip 1.5pt}
\cdashline{2-9}[4pt/2pt]
\noalign{\vskip 1.5pt}
 & KV4-RL & $92.4_{\pm 0.4}$ & $66.6_{\pm 1.0}$ & $65.7_{\pm 0.8}$ & $48.9_{\pm 6.9}$ & $41.1_{\pm 7.7}$ & $95.0_{\pm 2.5}$ & $68.3_{\pm 2.7}$ \\
\bottomrule
\end{tabular}
}
\endgroup

\end{table}

\subsection{Acceleration and Learned Performance on Qwen3-4B-Base}
\label{app:4b-continuations}

\begin{figure}[htbp]
  \centering
  \includegraphics[width=0.95\linewidth]{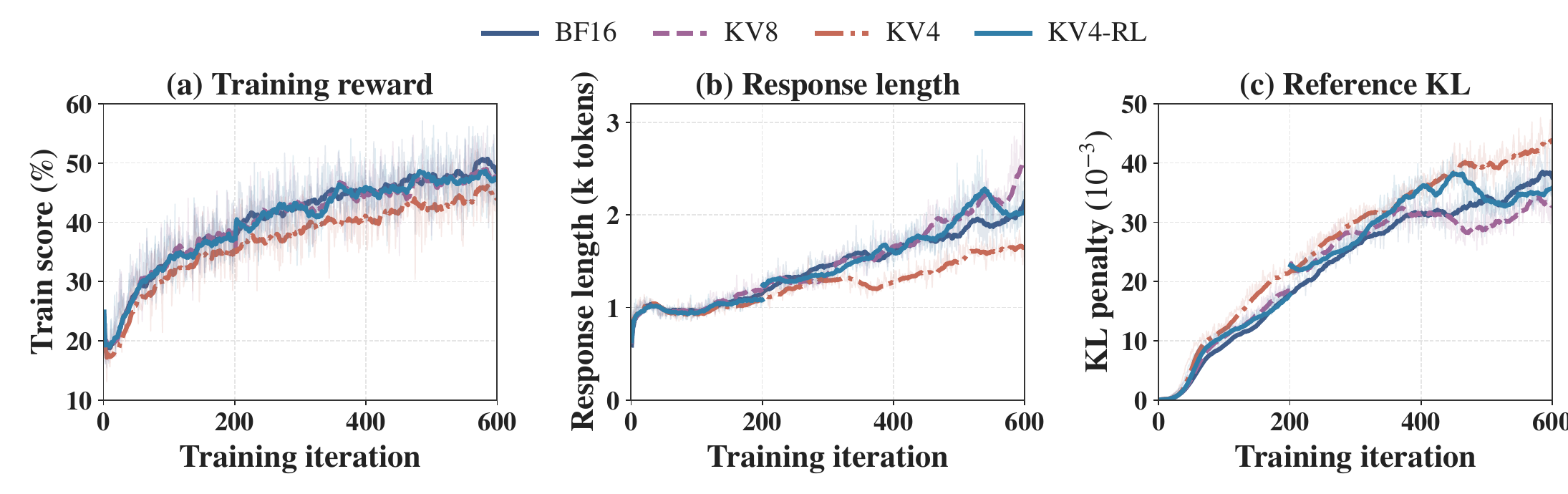}
  \vspace{-5pt}
  \caption{End-to-end timing for Qwen3-4B-Base over 600 training iterations. (a) Rollout collection time per iteration, with 20-iteration means computed separately within each uninterrupted logging segment. (b) Ratios of cumulative BF16 time to cumulative method time for rollout collection and full iterations. Cumulative ratios are computed from summed timers over the same iteration interval. KV4 is omitted from panel~(b) because its generated responses are substantially shorter. Timing excludes initialization, restart downtime, evaluation, and checkpoint saving.}
  \label{fig:app-4b-acceleration}
  \vspace{-10pt}
\end{figure}

Fig.~\ref{fig:app-4b-acceleration} shows that KV4-RL lowers collection time while retaining generation lengths close to BF16. Across the 600 recorded updates, its cumulative rollout time decreases by $1.57\times$ ($8.65\rightarrow5.50$ hours) and its cumulative full-iteration time by $1.17\times$ ($19.80\rightarrow16.87$ hours). The run-wide mean response lengths are 1.45k tokens for BF16 and 1.46k for KV4-RL, providing closely matched generated workloads. The rollout advantage grows to $1.72\times$ over the final 50 iterations, when both configurations average approximately 2.0k tokens per response. This trend is consistent with longer responses increasing both the amount of KV data processed per decode step and occupancy of the finite page pool. Unlike fixed-workload replay, the end-to-end measurement can therefore include both faster execution at a given decode state and additional concurrency enabled by the larger effective KV capacity, although it does not separately identify their contributions.

\begin{figure}[htbp]
  \centering
  \includegraphics[width=0.95\linewidth]{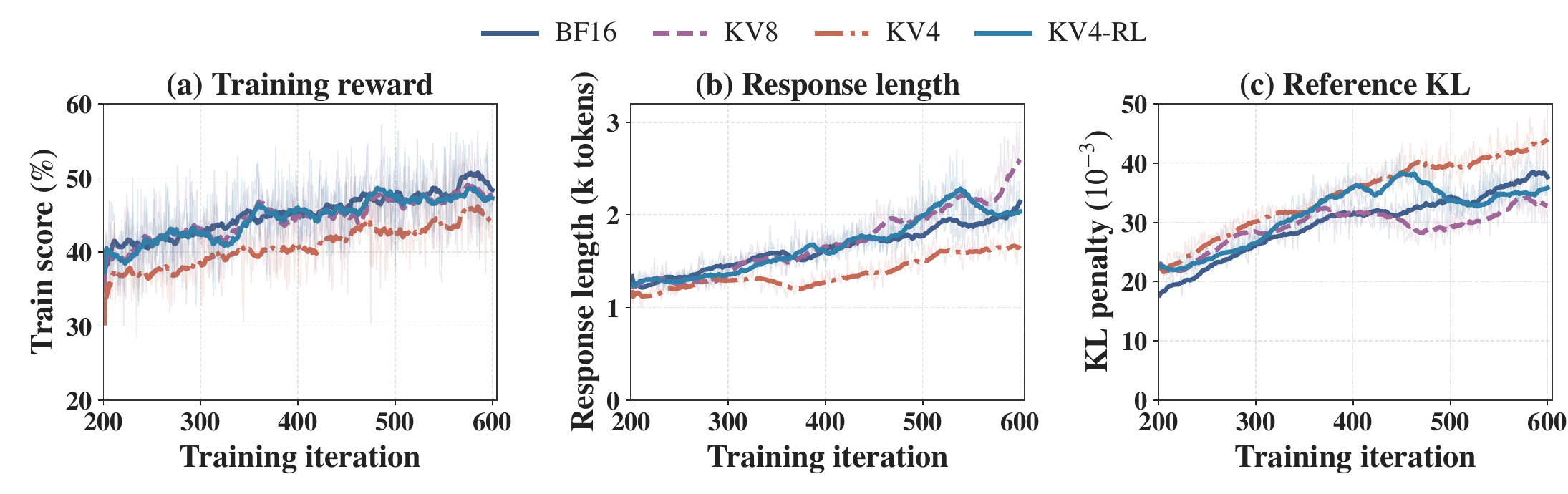}
  \vspace{-5pt}
  \caption{Training trajectories for Qwen3-4B-Base over 600 iterations. (a) Mean correctness reward. (b) Mean completed-response length. (c) Logged reference-KL penalty computed with the $k_3$ estimator against the initial policy. This statistic is distinct from the shared-prefix sampler KL in Eq.~\ref{eq:sampler-kl}. Thick curves show 20-iteration means, and faint curves show the corresponding raw traces.}
  \label{fig:app-4b-training}
  \vspace{-10pt}
\end{figure}

The 4B learning curves in Fig.~\ref{fig:app-4b-training} distinguish execution gains from workload changes. BF16 and KV4-RL show similar response-length growth and increasing reward. KV4 also lengthens its responses but remains shorter, earns lower reward, and has a larger reference-KL penalty. Over the final 100 updates, the reward gap to BF16 is 1.1 percentage points for KV4-RL and 4.6 points for KV4. As at 8B, collection time must therefore be interpreted together with the learned policy and generation workload.

\FloatBarrier
\subsection{Iteration-Time Breakdown}
\label{app:iteration-breakdown}

\begin{figure}[htbp]
  \centering
  \includegraphics[width=0.95\linewidth]{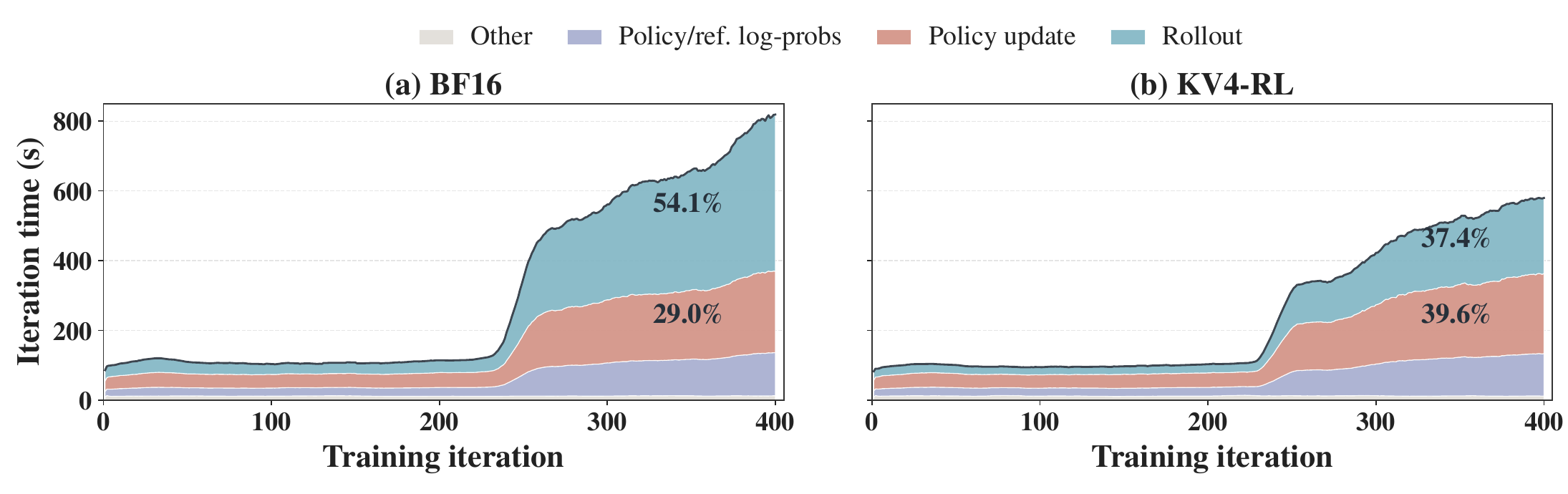}
  \vspace{-5pt}
  \caption{Iteration-time composition for BF16 and KV4-RL on Qwen3-8B-Base. Policy and reference log-probability computation are separate learner forward passes, while Other is the residual between the logged total and the displayed components. Curves show 20-iteration means, and the annotated percentages report mean component shares over the final 50 iterations.}
  \label{fig:iteration-breakdown}
  \vspace{-10pt}
\end{figure}

Fig.~\ref{fig:iteration-breakdown} shows how rollout acceleration redistributes the remaining iteration cost. Over the final 50 updates, KV4-RL reduces the rollout share from $54.1\%$ under BF16 to $37.4\%$. The policy-update share consequently rises from $29.0\%$ to $39.6\%$ and becomes the largest individual component. Including the policy and reference log-probability forward passes, learner-side computation accounts for $60.5\%$ of KV4-RL iteration time. These final-interval shares differ from the full-run BF16 rollout share reported in Section~\ref{sec:5.1} because response lengths and rollout costs evolve during training. The absolute learner cost remains similar across the two configurations, so it becomes the primary limit on further end-to-end gains from rollout-only acceleration.

\end{document}